\documentclass{article}

\usepackage{microtype}
\usepackage{graphicx}
\usepackage{booktabs} 

\usepackage{hyperref}

\usepackage{hos}
\usepackage{dsfont}

\usepackage{pgfplots}
\usepgfplotslibrary{fillbetween}
\usepackage{tikz}
\usetikzlibrary{bayesnet}
\usetikzlibrary{arrows}
\usepackage{subcaption}

\usepackage[accepted]{icml2019}

\icmltitlerunning{Learning Discrete and Continuous Factors of Data via Alternating Disentanglement}

\begin{document}

\twocolumn[
\icmltitle{Learning Discrete and Continuous Factors of Data\\ via Alternating Disentanglement}



\icmlsetsymbol{equal}{*}

\begin{icmlauthorlist}
\icmlauthor{Yeonwoo Jeong}{snu}
\icmlauthor{Hyun Oh Song}{snu}
\end{icmlauthorlist}

\icmlaffiliation{snu}{Department of Computer Science and Engineering, Seoul National University, Seoul, Korea}
\icmlcorrespondingauthor{Hyun Oh Song}{hyunoh@snu.ac.kr}

\icmlkeywords{Disentanglement, VAE}

\vskip 0.3in
]



\printAffiliationsAndNotice{}  

\begin{abstract}
We address the problem of unsupervised disentanglement of discrete and continuous explanatory factors of data. We first show a simple procedure for minimizing the total correlation of the continuous latent variables without having to use a discriminator network or perform importance sampling, via cascading the information flow in the $\beta$-vae framework. Furthermore, we propose a method which avoids offloading the entire burden of jointly modeling the continuous and discrete factors to the variational encoder by employing a separate discrete inference procedure.

This leads to an interesting alternating minimization problem which switches between finding the most likely discrete configuration given the continuous factors and updating the variational encoder based on the computed discrete factors. Experiments show that the proposed method clearly disentangles discrete factors and significantly outperforms current disentanglement methods based on the disentanglement score and inference network classification score. The source code is available at \href{https://github.com/snu-mllab/DisentanglementICML19}{https://github.com/snu-mllab/DisentanglementICML19}.
\end{abstract}
\section{Introduction}\label{sec:intro}

Learning to disentangle the underlying explanatory factors of data without supervision is a crucial task for representation learning in AI-related tasks such as speech, object recognition, natural language processing, and transfer learning \cite{bengio13}. While establishing a clear quantifiable objective is difficult, in a successfully disentangled representation, a single latent unit of the representation should correspond to a change in a single generative factor of the data while being relatively invariant to others.

To this end, a line of research on unsupervised disentanglement has been pursued under the Variational Autoencoder framework \cite{betavae,factorvae,betatc,anchorvae}. The common theme among the recently proposed methods is to penalize the \emph{total correlation} among the latent variables so that the model is encouraged to learn statistically independent factors of data. 

Although penalizing the total correlation is important we argue that this alone is not sufficient for learning disentangled representations. Most recent methods focus on learning only the continuous factors of variation and jointly modeling both the continuous and discrete factors of variation is relatively much less studied. When modeling complex and high-dimensional data such as raw images, it becomes difficult to disentangle the discrete factors of data (\ie~ number of light sources, categorical shape of present objects) from continuous factors (\ie~ translation, rotation, color) under these methods. \citet{aae} have demonstrated that providing the true discrete factors of data to the autoencoder via supervision drastically improves the quality of the learned continuous factors compared to the purely unsupervised case. We hypothesize that lumping both the continuous and discrete factors into a single latent vector and optimizing for the joint variational posterior severely deteriorates the disentanglement performance as this imposes too much modeling burden to the posterior.

In this paper, we first propose a simple procedure for penalizing the total correlation which does not require any extra discriminator network or having to run expensive importance sampling in the $\beta$-VAE framework. Then, we propose an alternating disentanglement method where it alternates between finding the most likely configuration of the discrete factors given the continuous factors and updating the inference parameters given the discrete configuration. The empirical results show that decoupling the disentanglement process for continuous and discrete factors via the proposed alternating method leads to strong disentanglement performance both qualitatively and quantitatively.
 
Our quantitative results on 1) dSprites \cite{dsprites} dataset on the disentanglement evaluation metric by \citep{factorvae}, and on 2) the inference network classification score on the learned discrete factors show state of the art results outperforming recently proposed disentanglement methods: $\beta-$VAE, AnchorVAE, FactorVAE, and JointVAE by a large margin.

\section{$\beta$-VAE and disentanglement}\label{sec:prelim}
We first review and analyze how the $\beta$-VAE framework  relates to disentanglement from information theoretic perspective. VAE is a latent variable model that pairs a top-down decoding generator ($\theta$) and a bottom-up encoding inference network ($\phi$). Then, a variational lower bound of the marginal log-likelihood, $\mathbb{E}_{x \sim p(x)} \log p(x)$, is maximized. Concretely, the VAE objective is,
\small
\begin{align}
\label{eqn:vae}
\mathcal{L}(\theta,\phi) = \mathbb{E}_{x \sim p(x)} \big[&\mathbb{E}_{z \sim q_\phi(\cdot \mid x)} \log p_\theta(x \mid z)\\
&- \beta D_\text{KL}\left(q_\phi(z \mid x) \parallel p(z)\right)\big],\nonumber
\end{align}
\normalsize
where $p(z)$ is the fully factorized standard normal prior. Note, maximizing the objective in \Cref{eqn:vae} can be viewed as maximizing the lower bound on the mutual information between the data $x$ and the latent code $z$ with the KL term. Concretely,
\small
\begin{align}
\label{eqn:mi_beta_upb}
    &I(x;z) - \beta \mathbb{E}_{x\sim p(x)}D_\text{KL}(q_\phi(z\mid x) \parallel p(z))\\
    &=H(x)+\int q_\phi(x,z) \log q_{\phi}(x \mid z) dz dx\nonumber \\
    &\quad- \beta \mathbb{E}_{x\sim p(x)}D_\text{KL}(q_\phi(z\mid x) \parallel p(z))\nonumber\\
    &\geq H(x)+\mathbb{E}_{x\sim p(x)}\big[\mathbb{E}_{z\sim q_\phi(\cdot \mid x)}\log p_{\theta}(x\mid z)\nonumber\\
    &\quad- \beta D_\text{KL}(q_\phi(z\mid x)\parallel p(z))\big],\nonumber
\end{align}
\normalsize
where the data entropy can be ignored as there is no dependence on the parameters. Also, the KL term in \Cref{eqn:vae,eqn:mi_beta_upb} can be factorized as below \cite{factorvae,betatc},
\small
\begin{align}
\label{eqn:kl_factors}
&\mathbb{E}_{x \sim p(x)} D_\text{KL}(q(z| x) ~||~ p(z))=I(x;z) + D_\text{KL}(q(z) || p(z))\\ 
&=I(x;z) + \underbrace{D_\text{KL}(q(z) \parallel \prod_j q(z_j))}_{=~\text{Total correlation},~ TC(z)} + \sum_j D_\text{KL}(q(z_j) \parallel p(z_j)),\nonumber
\end{align}
\normalsize
where $q(z)$ denotes the marginal posterior computed as $q(z) = \mathbb{E}_{x\sim p(x)} q(z\mid x)$. The second term in the factorization is known as \emph{total correlation} (TC) and is a popular measure quantifying the redundancy among a set of $m$ random variables \cite{watanabe1960}. The significance of TC is that penalizing TC causes the model to learn statistically independent factors in the data which is a crucial component in disentangled representations \cite{bengio13}.

\section{Total correlation penalization with information cascading}\label{sec:cascade}

$\beta$-VAE indirectly penalizes $TC(z)$ by increasing the $\beta$ coefficient to a high value in the KL divergence term in \Cref{eqn:vae} \cite{betavae}. However, this inevitably penalizes the $I(x;z)$ term altogether (see \Cref{eqn:kl_factors}) leading to reduction in the amount of information in $z$ about $x$. To address this, FactorVAE decreases $\beta$ and introduces an additional regularization term between the marginal variational posterior and factorized marginal $D_\text{KL}(q(z) \parallel \prod_i^m q(z_i))$ using the density-ratio trick via a separate discriminator network \cite{factorvae}. $\beta$-TCVAE on the otherhand estimates TC via importance sampling within minibatches \cite{betatc}. We first make the following observations and show a simple practical procedure for minimizing TC without having to rely on additional neural networks or sampling procedures in the $\beta$-VAE framework.
\begin{proposition}
\label{prop1}
The mutual information between a single random variable and the rest can be factorized as
\[I(z_{1:i-1}; z_i) = TC(z_{1:i}) - TC(z_{1:i-1})\]
\end{proposition}
\begin{proof}
See supplementary A1. 
\end{proof}
\begin{proposition}
\label{prop2}
The mutual information between $x$ and partitions of $z = [z_1, z_2]$ can be factorized as,
\[I(x; [z_1, z_2]) = I(x;z_1) + I(x; z_2) - I(z_1; z_2)\]
\end{proposition}
\begin{proof}
See supplementary A2. 
\end{proof}
Now, by telescoping sum, we can write,
\small
\begin{align}
\label{eqn:tc_factors}
TC(z) &= \underbrace{TC(z_{1:2})}_{= I(z_1; z_2)} + \sum_{i=3}^m \left( \underbrace{TC(z_{1:i}) - TC(z_{1:i-1})}_{= I(z_{1:i-1}; z_i)} \right)\nonumber\\
&= \sum_{i=2}^m I(z_{1:i-1}; z_i)
\end{align}
\normalsize
This is because the first term in \Cref{eqn:tc_factors} is equal to $I(z_1; z_2)$ by definition of mutual information and the rest of the terms are equal to $I(z_{1:i-1}; z_i)$ by \Cref{prop1}. Now we aim at penalizing $TC(z)$ by sequentially penalizing the individual summand in \Cref{eqn:tc_factors}. From \Cref{prop2}, we can write 
\[I(x;z_{1:i}) = I(x;z_{1:i-1}) + I(x;z_{i}) - I(z_{1:i-1};z_{i})\]
This factorization motivates a maximization algorithm sequentially updating the left hand side $I(x; z_{1:i})$ for all $i=2,\ldots,m$ which in turn minimizes each summand in \Cref{eqn:tc_factors}. Also, from the lower bound of mutual information in \Cref{eqn:mi_beta_upb} we have, 
\[I(x;z_{1:i}) \geq H(x) + \mathbb{E}_x \mathbb{E}_{z_{1:i}\sim q_\phi(\cdot \mid x)} \log p_\theta(x \mid z_{1:i})\]
We maximize $I(x;z_{1:i})$ by maximizing its lower bound $\mathbb{E}_x \mathbb{E}_{z_{1:i}\sim q_\phi(\cdot|x)}p_\theta(x|z_{1:i})$. In practice, we observed it is sufficient to maximize the objective in \Cref{eqn:vae} while  penalizing $z_{i+1:m}$ with a large beta coefficient on $D_\text{KL}(q_{\phi}(z_{i+1:m}\mid x) \parallel p(z_{i+1:m}))$\footnote{Note, under conditional independence, $\beta D_\text{KL}(q_\phi(z_{i+1:m}\mid x) \parallel p(z_{i+1:m})) = \sum_{j=i+1}^m \beta D_\text{KL}(q_\phi(z_j\mid x) \parallel p(z_j))$}. This leads to sequential updates where each latent dimensions are heavily penalized with high $\beta$ in the beginning but are sequentially relieved one at a time with small $\beta$ in a cascading fashion. 

Since this procedure implicitly penalizes the individual summand $I(z_{1:i-1},z_i)$ in the factorization of TC, we can interpret that the method separately controls the information flow one latent variable at a time encouraging the model to encode one statistically independent factor of information per each newly \emph{opened} variable.

To contrast the procedure with related methods, \citet{burgess17} and JointVAE increases the KL capacity term which is \emph{shared} across all latent variables. AnchorVAE sets $\beta$ to a small value for a subset of variables and high value for the rest of the variables but does not dynamically control the information flow \cite{anchorvae}.

\section{Alternating disentanglement of discrete and continuous factors}

\Cref{fig:graphical_models} illustrates the graphical models view of the generative and inference processes for $\beta$-VAE \cite{betavae}, JointVAE \cite{jointvae}, AAE with supervised discrete factors \cite{aae}, and our proposed model respectively. As illustrated in \Cref{fig:graph_jointvae}, JointVAE can be viewed as  augmenting the continuous latent variables with discrete latent variables ($z=[z'; d]$) in the $\beta$-VAE framework. However, simply lumping the latent variables together offloads the entire burden of jointly modeling both the continuous and discrete factors to the variational posterior which can be very challenging.

\begin{figure}[t]
\centering %
\begin{subfigure}[t]{0.23\linewidth}
\centering
\resizebox*{!}{\linewidth}{
\begin{tikzpicture}
    \node[obs] (x) {$x$};%
    \node[latent,above=of x] (z) {$z$}; %
    \edge {z} {x}; %
    \path (x) edge [dashed, bend left=40,->] (z); 
\end{tikzpicture}}
\caption{$\beta$-VAE}
\label{fig:graph_vae}
\end{subfigure}\hspace{-0.5em}\hfill
\begin{subfigure}[t]{0.23\linewidth}
\centering
\resizebox*{!}{\linewidth}{
\begin{tikzpicture}
    \node[obs]                               (x) {$x$};
    \node[latent, above=of x, xshift=-0.8cm] (z) {$z$};
    \node[latent, above=of x, xshift=0.8cm]  (d) {$d$};
    \edge {z,d} {x} ; %
    \path (x) edge [dashed, bend left=40,->] (z);
    \path (x) edge [dashed, bend right=40,->] (d);
\end{tikzpicture}}
\caption{JointVAE}
\label{fig:graph_jointvae}
\end{subfigure}\hspace{1em}\hfill
\begin{subfigure}[t]{0.23\linewidth}
\centering
\resizebox*{!}{\linewidth}{
\begin{tikzpicture}
    \node[obs]                               (x) {$x$};
    \node[latent, above=of x, xshift=-0.8cm] (z) {$z$};
    \node[obs, above=of x, xshift=0.8cm]  (d) {$d$};
    \edge {z,d} {x} ; %
    \path (x) edge [dashed, bend left=40,->] (z);
\end{tikzpicture}}
\caption{AAE-S}
\label{fig:graph_aae}
\end{subfigure}\hspace{1em}\hfill
\begin{subfigure}[t]{0.23\linewidth}
\centering
\resizebox*{!}{\linewidth}{
\begin{tikzpicture}
    \node[obs]                               (x) {$x$};
    \node[latent, above=of x, xshift=-0.8cm] (z) {$z$};
    \node[latent, above=of x, xshift=0.8cm]  (d) {$d$};
    \edge {z,d} {x} ; %
    \path (x) edge [dashed, bend left=40,->] (z);
\end{tikzpicture}}
\caption{Ours}
\label{fig:graph_ours}
\end{subfigure}\hfill
\caption{Graphical models view of $\beta$-VAE, JointVAE, AAE with supervised discrete variables, and our method. Solid lines denote the generative process and the dashed lines denote the inference process. $x,z,d$ denotes the data, continuous latent code, and the discrete latent code respectively.}
\label{fig:graphical_models}
\end{figure}
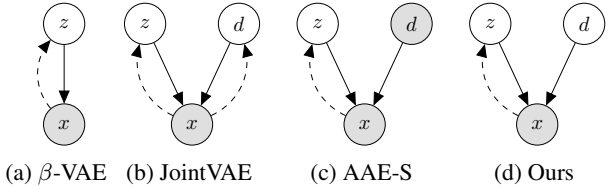

\citet{aae} have investigated learning continuous latent representations while providing the discrete factors through supervision (\ie~providing class labels in MNIST) and demonstrated that the model can learn drastically better continuous representations when the burden of simultaneously modeling the continuous and discrete factors is relieved. Inspired by these findings, our idea is to alternate between finding the most likely discrete configuration of the variables given the continuous factors and updating the parameters ($\phi,\theta$) given the discrete configurations. 

\Cref{fig:graph_aae} and \Cref{fig:graph_ours} illustrate that the generative and inference process between the two methods resemble in the sense that they only encode the continuous factor $z$, while \Cref{fig:graph_jointvae} encodes both the continuous and discrete factors $z, d$ simultaneously. Unlike AAE-S, discrete factors are not observed in our case. 

The joint distributions $q(x,z,d)$ of JointVAE and AAE-S are
\small
\begin{align*}
    q_\phi(x,z,d) = p(x)q_\phi(z,d|x)
\end{align*}
\normalsize
and 
\small
\begin{align*}
    q_\phi(x,z,d)=\begin{cases} q_\phi(x,z) & \text{if}~d=y \\ 0 & \text{otherwise} \end{cases},
\end{align*}
\normalsize
where $y$ is the provided ground truth discrete factors, respectively. We define the joint distribution similar to AAE-S as 
\small
\begin{align*}
    q_{\phi,\theta}(x,z,d)=\begin{cases} q_\phi(x,z) & \text{if}~d=\argmax_{\hat{d}} p_\theta(x \mid z,\hat{d})\\ 0 & \text{otherwise} \end{cases},
\end{align*}
\normalsize
where $q_\phi(x,z)=p(x)q_\phi(z\mid x)$. Likewise, the variational posterior is defined as,
\small
\begin{align*}
    q_{\theta,\phi}(z,d\mid x)=\begin{cases} q_\phi(z\mid x) & \text{if}~d=\argmax_{\hat{d}} p_\theta(x\mid z,\hat{d})\\ 0 & \text{otherwise} \end{cases}.
\end{align*}
\normalsize
Note in our case the inference for the discrete factors involves both the encoder ($\phi$) and the decoder ($\theta$) in contrast to JointVAE. 

Note from the factorization in \Cref{eqn:kl_factors},  the KL term in $\beta$-VAE is factorized as the sum of the mutual information term and the divergence from the marginal posterior to the prior. Since the mutual information for discrete variables, $I(x,d)$ is bounded above by $H(d)$, we only need to consider the prior divergence term. For discrete uniform prior, the following lemma shows a useful upperbound which is easier to optimize directly\footnote{This will be apparent in the following subsection.}.
\begin{lemma}
    If $p(d_i)$ is the discrete uniform distribution supported on a finite set of cardinality $S_i$, $D_{KL}(q(d)||p(d)) \leq S\mathbb{E}_{d, d' \sim q(d)}[\mathds{1}(d=d')]- 1$ where $S=\prod_i S_i$. 
\end{lemma}
\begin{proof}
    Denote $p(d)=\prod_i p(d_i)$ which is also discrete uniform.
    \small
    \begin{align*}
        &D_{KL}(q(d)||p(d)) = \sum_d q(d) \log \frac{q(d)}{p(d)}\leq \sum_d q(d) \left( \frac{q(d)}{p(d)} - 1\right) \\
        &= S \sum_d q^2(d) - 1 = S \mathbb{E}_{d, d' \sim q(d)}[\mathds{1}(d=d')] - 1
    \end{align*}
\end{proof}
\normalsize
Now our goal is to maximize the following objective
\small
\begin{align}
\label{eqn:objective_original}
&\mathcal{L}(\theta,\phi)=I(x; [z,d]) - \beta \mathbb{E}_{x\sim p(x)} D_\text{KL}(q_\phi(z \mid x) \parallel p(z))\nonumber\\
&\quad\quad\qquad ~~-\lambda D_\text{KL} (q(d) \parallel p(d))\nonumber\\
& \geq H(x)+\int \sum_d q(x,z,d) \log q_{\theta, \phi}(x|z,d) dz dx\nonumber\\
&\quad-\beta \mathbb{E}_{x\sim p(x)} D_\text{KL}(q_\phi(z \mid x) \parallel p(z)) - \lambda D_\text{KL} (q(d) \parallel p(d))\nonumber\\
&\geq H(x)+\mathbb{E}_{x\sim p(x)}[\mathbb{E}_{z,d\sim q_{\phi, \theta}(\cdot | x)}[\log p_{\theta}(x|z,d)]]\nonumber\\
&\quad-\beta \mathbb{E}_{x\sim p(x)} D_\text{KL}(q_\phi(z \mid x) \parallel p(z))\nonumber\\
&\quad- \lambda (S\mathbb{E}_{d, d' \sim q(d)}[\mathds{1}(d=d')]-1)
\end{align}
\normalsize

\subsection{Alternating minimization scheme}
Note the lower bound objective in \Cref{eqn:objective_original} has an inner maximization step over the discrete factors embedded in the variational posterior. Suppose the data is sampled $x^{(i)} \in \mathcal{X}, i=1, \ldots, n$, the continuous latent variables are samples from the decoder $q_\phi(\cdot \mid x^{(i)})$, and the discrete latent variables are represented using one-hot encodings of each variables $d^{(i)} \in \{e_1, \ldots, e_S\}$. After rearranging the terms, we arrive at the following optimization problem.
\small
\begin{align}
\label{eqn:master_eqn}
    &\maximize_{\theta, \phi} \left(\underbrace{\maximize_{d^{(1)},\ldots d^{(n)}} \sum_{i=1}^n {u_\theta^{(i)}}^\intercal d^{(i)} - \lambda' \sum_{i\neq j} {d^{(i)}}^\intercal d^{(j)}}_{:= \mathcal{L}_{LB}(\theta, \phi)} \right)\nonumber\\
    &\qquad\qquad\quad - \beta \sum_{i=1}^n D_{KL}(q_\phi(z|x^{(i)})||p(z))\nonumber\\
    &\text{\ \  subject to }~~ \| d^{(i)} \|_1 = 1,~ d^{(i)} \in \{0,1\}^S,~ \forall i, 
\end{align}
\normalsize
where $u_\theta^{(i)}$ denotes the vector of the  likelihood $\log p_\theta(x^{(i)}|z^{(i)}, e_k)$ evaluated at each $k \in [S]$. Note, the inner maximization problem $\mathcal{L}_{LB}(\theta, \phi)$ in \Cref{eqn:master_eqn} over the discrete variables $[d^{(1)},\ldots,d^{(n)}]$ subject to the sparsity equality constraints can be exactly solved in polynomial time via \emph{minimum cost flow} without  continuous relaxation as shown in Theorem 1 in \cite{JeongS18}.

\section{Related works}\label{sec:rel}
Unsupervised discovery of disentangled factors dates back to 90s. \citet{schumidhuber1992} penalizes the predictability of a latent dimension given the others but the approach did not scale very well. More recently, VAE was proposed and the framework offered scalability and the optimization stability \cite{vae}. Then \citet{betavae} showed that tuning the $\beta$ hyperparameter in VAE to $\beta > 1$ can influence the model to learn statistically independent and disentangled representations by limiting the capacity of the latent information channel. 

The recent follow up works from \citet{factorvae,betatc,anchorvae} then analyzed the KL divergence term under the expectation over the data distribution could be broken down into the mutual information term between the data and the latent code, and the KL divergence term between the latent distribution and the factorial prior often denoted as total correlation (TC) \cite{watanabe1960}. The idea was that regularizing the KL divergence term with high $\beta$ in the VAE framework not only penalizes for TC but also inevitably penalizes the mutual information between the data and the latent variables. 

To this end, the recent works proposed decreasing $\beta$ but more explicitly penalizing for TC. FactorVAE estimates the total correlation via density ratio trick which utilizes a discriminator network \cite{factorvae}. $\beta$-TCVAE employs mini-batch weighted sampling to estimate TC \cite{betatc}. On the contrary, we observe a factorization $TC(z) = \sum_{i=2}^m I(z_{1:i-1}; z_i)$ and show we can penalize TC without any explicit computation by incrementally penalizing each summand in the factorization by cascading the information flow.

On the other hand, NVIL \cite{nvil} and VIMCO \cite{vimco} have explored training VAEs with only discrete latent variables via REINFORCE \cite{williams92} with variance reduction techniques. VQVAE \cite{vqvae} learns discrete latent variables via vector quantization. However,  modeling via binary latent variables alone can be inappropriate since the underlying modalities would be a mix of both continuous and discrete factors for high dimensional complex data.

Jointly modeling the continuous and discrete generative latent factors has been much less explored. InfoGAN \cite{infogan} models both the factors and is based on Generative Adversarial Network (GAN) framework \cite{goodfellow14}. InfoGAN aims at disentangling the factors by maximizing the mutual information between a subset of latent dimensions and the generated samples. However, \cite{factorvae} showed the learning process can be very unstable and significantly reduces the sample diversity in contrast to the $\beta$-VAE framework \cite{betavae}. Empirically, InfoGAN tends to mix discrete and continuous factors which results in lower disentanglement score \cite{factorvae} than $\beta$-VAE based methods. JointVAE proposed jointly modeling both the continuous and discrete factors by augmenting the continuous and discrete latent variables. This introduces an additional KL divergence term for discrete latent variables which is optimized by a continuous reparameterization via Gumbel softmax trick \cite{gumbel1954,maddison16,jang16}. Our method, on the other hand, decouples the task of jointly modeling the continuous and discrete factors of data via alternating maximization and shows significant gains in the disentanglement performance.

\section{Implementation details}
Model architecture and training details are provided in supplementary B. As discussed in \Cref{sec:cascade}, we individually control the $\beta$ term on each continuous variables. Let $\beta_j$ denote the coefficient for a variable $j$. Each $\beta_j$'s start at the high value $\beta_h$ and gets relieved one at a time to the low value $\beta_l$. After each $r$ iterations, we relieve one variable $j$ by switching the coefficient from $\beta_h$ to $\beta_l$. The alternating maximization with discrete variables is enabled after a warm-up time denoted as $t_d$.

For ablation study, we denote our method without the discrete variables as CascadeVAE-C, our method including the discrete variables but without the information cascading as CascadeVAE-D, and the full method as CascadeVAE. \Cref{alg:CascadeVAE-C} shows the pseudocode for CascadeVAE. Note, if $t_d$ is greater than MAXITER, then the pseudocode trains CascadeVAE-C. 

\begin{algorithm}[h!]
   \caption{CascadeVAE}
	\label{alg:CascadeVAE-C}
\begin{algorithmic}
    \STATE \textbf{Input : } Data $\{x^{(i)}\}_{i=1}^N$, Encoder($q_\phi$), Decoder($p_\theta$), $\beta_l, \beta_h$, $r, t_d$, optimizer $g$
    \STATE Initialize parameters $\phi, \theta$.
    \STATE Set $\beta_j=\beta_h,~$ $\forall j$ and $d^{(i)}=0,$ $~~\forall i\in [N]$
    \STATE Set $j=1$.
    \FOR{$t=1,\ldots,$MAXITER}
    \STATE \textbf{if} $t$ is a multiple of $r$
    \STATE $~~~$Switch $\beta_j$ to $\beta_l$ and $j\leftarrow j+1$
    \STATE Randomly select batch $\{x^{(i)}\}_{i \in \mathcal{B}}$
    \STATE Sample $z^{(i)}_\phi \sim q_{\phi}(z|x^{(i)})$ $\forall i \in \mathcal{B}$ 
    \STATE \textbf{if} $t>t_d$
    \STATE $~~~$ Update $u_\theta^{(i)}$ by computing $\log p_\theta(x^{(i)}|z^{(i)}, e_k)$ $\forall k$
    \STATE $~~~$ Compute $\mathcal{L}_{LB}(\theta, \phi)$ by solving for the optimal \\
    \qquad assignment $\{d^{(i)}\}_{i\in \mathcal{B}}$ via minimum cost flow 
    \STATE $\theta, \phi \leftarrow g\left(\nabla_{\theta,\phi} \mathcal{L}_{LB}(\theta,\phi)\right)$
    \ENDFOR
\end{algorithmic}
\end{algorithm}

\section{Experiments}\label{sec:exp}
We perform experiments on dSprites \cite{dsprites}, MNIST, FashionMNIST \cite{fmnist}, and Chairs \cite{chairs} datasets. For quantitative results, dSprites dataset comes labeled with the generative factors which allow quantitative comparisons on the disentanglement metrics. We additionally report unsupervised classification accuracy using the learned discrete variables from running inference on dSprites, MNIST, and FashionMnist.

\subsection{Experiments on dSprites}
DSprites has $737,280$ images of size $64\times 64$ with 5 generative factors: shape (3), scale (6), orientation (40), x-position (32), and y-position (32). We evaluate the performance with the disentanglement score metric proposed by \citet{factorvae}. The details on disentanglement score are provided in Supplementary C. \Cref{tab:dsprites_score} compares the disentanglement scores for various baselines. For AnchorVAE, we experimented anchoring 5 latent dimensions out of total 5 and 20 dimensions. For FactorVAE, we performed hyperparameter search over both $\beta$ and $\gamma$ to reproduce the reported results from the paper. The dimension of discrete latent representation $S$, is fixed to $3$ following the experiment protocol in JointVAE for a fair comparison.

The results for CascadeVAE-C show that cascading the information flow alone in the $\beta$-VAE framework achieves competitive disentanglement scores to the current state of the art FactorVAE method. CascadeVAE-D which does not penalize for TC via information cascading also performs well and boosts the performance of $\beta$-VAE by up to 10 points on the disentanglement metric. 

Our full method without ablations is denoted as CascadeVAE (the last row) in \Cref{tab:dsprites_score}. The method shows approximately 10 points boost on top of the current state of the art FactorVAE method. The experiments suggest that both the information cascading for implicit TC penalization and the discrete modeling via alternating disentanglement have complementary benefits leading to a significant improvement over the baseline methods.

\begin{table}[htbp]
\centering
\fontsize{9pt}{9.5pt}\selectfont
\begin{tabular}{rc rr}
\addlinespace[-\aboverulesep]
\toprule
\multicolumn{1}{c}{Method}&m&\multicolumn{1}{c}{Mean (std)}&\multicolumn{1}{c}{Best}\\
\toprule
$\beta$ VAE\\($\beta=10.0$)& 5 & 70.11 (7.54)&84.62\\
($\beta=4.0$)& 10 &74.41 (7.68)&88.38\\
\midrule
AnchorVAE\\($\beta_l=10.0$)& 5 (5) & 76.36 (4.96)&82.75\\
       ($\beta_l=7.0$)& 5 (20)& 72.44 (6.85)&83.25\\
\midrule
FactorVAE& 5  &81.09 (2.63)&85.12\\
        & 10 &82.15 (0.88)&88.25\\
\midrule
CascadeVAE-C\\($\beta_l=0.7$)& 5  &81.69 (3.14)&88.38\\
    ($\beta_l=1.0$)& 10 &81.74 (2.97)&87.38\\
\midrule
\midrule
JointVAE & 6 &  74.51 (5.17)&91.75\\
        & 4 &  73.06 (2.18)&75.38\\
\midrule
CascadeVAE-D\\($\beta=7.0$) & 6 &79.67 (5.36) & 90.25\\
             ($\beta=1.0$)& 4 &80.70 (4.77) & 96.50\\
\midrule
CascadeVAE\\($\beta_l=1.0$)& 6 &90.49 (5.28)&\textbf{99.50}\\
($\beta_l=2.0$)& 4 &\textbf{91.34 (7.36)}&98.62\\
\bottomrule
\end{tabular}
\caption{DSprites disentanglement score for various baselines. The score is obtained from 10 different random seed each with the best hyperparameters.}
\label{tab:dsprites_score}
\end{table}

\Cref{fig:dsprites_tc_score} shows the scatter plot of $TC(z)$ versus the disentanglement scores for 10 different random seeds. We chose the best hyperparameter settings for each method. The results for CascadeVAE-C show comparable $TC(z)$ with $\beta$-VAE even though the $\beta$ coefficient for our case is much smaller ($\beta_l=1.0$ for CascadeVAE-C and $\beta=4.0$ for $\beta$-VAE). This also shows we can effectively penalize TC implicitly by information cascade without adding extra discriminator networks for explicit penalization while preserving the disentanglement performance. Notably, CascadeVAE shows smallest TC while showing the highest disentanglement performance, again confirming our hypothesis that both TC and joint modeling of discrete and continuous factors are essential in learning disentangled representations.

\begin{figure}
\begin{tikzpicture}
\begin{axis}[ only marks, grid=major, scaled ticks = false, ylabel near ticks, tick pos = left, 
    tick label style={font=\small}, 
    ytick={0.5, 0.6, 0.7, 0.8, 0.9}, 
    xtick={1, 2, 3, 4, 5, 6, 7, 8,9.5,10.5,11.5}, 
    xticklabels={1, 2,3, 4, 5, 6, 7, 8, 18,24,30}, 
    label style={font=\small}, xlabel={$TC(z)$}, ylabel={Disentanglement score}, xmin=0, xmax=12.5, ymin=0.6, ymax=1.0, legend style={legend columns=1, font=\scriptsize}, legend cell align={left}, legend pos=north east]
\addlegendentry{CascadeVAE ($\beta_l=2.0$)}
\addplot+[blue, mark options={fill=blue, scale=0.7, mark=*, solid}] table  [x=tc, y=score, col sep=comma]{csv/tc_score_augvae_discrete.csv};
\addlegendentry{CascadeVAE-C ($\beta_l=1.0$)}
\addplot+[red, mark options={fill=red, scale=0.7, mark=square*, solid}] table  [x=tc, y=score, col sep=comma]{csv/tc_score_augvae.csv};
\addlegendentry{$\beta$ VAE ($\beta=4.0$)}
\addplot+[cyan, mark options={fill=cyan, scale=0.7, mark=diamond*, solid}] table  [x=tc, y=score, col sep=comma]{csv/tc_score_vae.csv};
\addlegendentry{FactorVAE ($\beta=0.4, \gamma=60.0$)}
\addplot+[black, mark options={fill=black, scale=0.7, mark=star, solid}] table  [x=tc, y=score, col sep=comma]{csv/tc_score_factorvae.csv};
\addlegendentry{JointVAE}
\addplot+[green, mark options={fill=green, scale=0.7, mark=triangle*, solid}] table  [x=tc, y=score, col sep=comma]{csv/tc_score_jointvae.csv};
\end{axis}
    \draw (4.8, 0.0) -- node[fill=white,inner sep=-1.25pt,outer sep=0,anchor=center]{\rotatebox[origin=c]{90}{$\approx$}} (4.8, 0.0);
    \draw (4.8, 5.7) -- node[fill=white,inner sep=-1.25pt,outer sep=0,anchor=center]{\rotatebox[origin=c]{90}{$\approx$}} (4.8, 5.7);
\end{tikzpicture}
\caption{$TC(z)$ and disentanglement score on 10 different random seed on dSprites dataset}
\label{fig:dsprites_tc_score}
\end{figure}
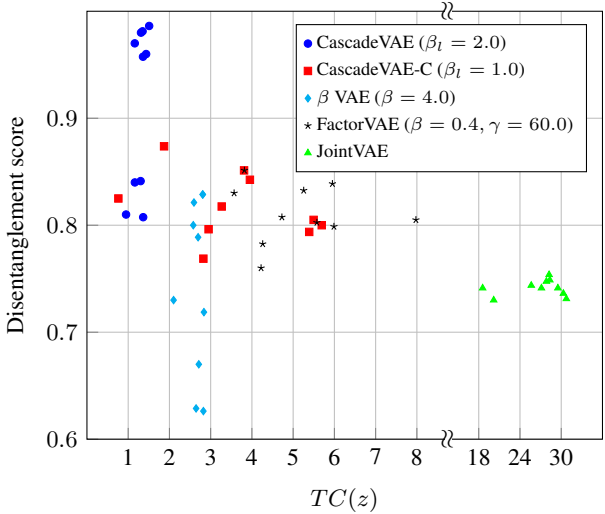

In dSprites dataset, the discrete factors encode the categorical shape information (ellipse, heart, square). Following the experiment protocol in JointVAE, we evaluated the classification accuracy computed from the discrete variables from inference. \Cref{tab:dsprites_cls} compares the results of our method against JointVAE. Note, other baseline methods do not jointly model the continuous and discrete variables. The results show about $30\%$ difference in the classification accuracy. Notably, at the best run, we achieve $99.66\%$ classification accuracy even though the method is unsupervised.

\begin{table}[htbp]
\fontsize{9pt}{9.5pt}\selectfont
\centering
\begin{tabular}{rc ll}
\addlinespace[-\aboverulesep]
\toprule
\multicolumn{1}{c}{Method}&m&\multicolumn{1}{c}{Mean (std)}&\multicolumn{1}{c}{Best}\\
\toprule
JointVAE & 6&44.79 (3.88) &53.14\\
        &  4&43.99 (3.94) &54.11\\
\midrule
CascadeVAE &6&\textbf{78.84 (15.65)}& \textbf{99.66}\\
                &4&76.00 (22.16)& 98.72\\
\bottomrule
\end{tabular}
\caption{Unsupervised classification results on dSprites, $S=3$. Unsupervised classification accuray for random chance is $33.33$.}
\label{tab:dsprites_cls}
\end{table}

For qualitative results, \Cref{fig:dsprites_latent_traversal_beta_factor} first shows the latent traversal results from $\beta$-VAE and FactorVAE. The results show that the methods are capable of capturing the $x,y$-positions, and scale factors but does not disentangle the orientation and shape factors of variation clearly. \Cref{fig:dsprites_latent_traversal_joint_ours} compares the discrete traversal results against JointVAE. Even though we chose the best runs out of 10 random seeds for both methods, JointVAE does not clearly disentangle the discrete shapes. In contrast, CascadeVAE shows almost perfect disentanglement of the discrete shapes where the discrete code $[1 0 0], [0 1 0],$ and $[0 0 1]$ correspond to the ellipse, heart, and square categories respectively. \Cref{fig:dsprites_plot} (left) shows the iteration versus mutual information of each latent dimension $I(x;z_i)$ plot when CascadeVAE is trained with 6 continuous and 1 discrete variables. \Cref{fig:dsprites_plot} (right) shows the traversal results for each variables.

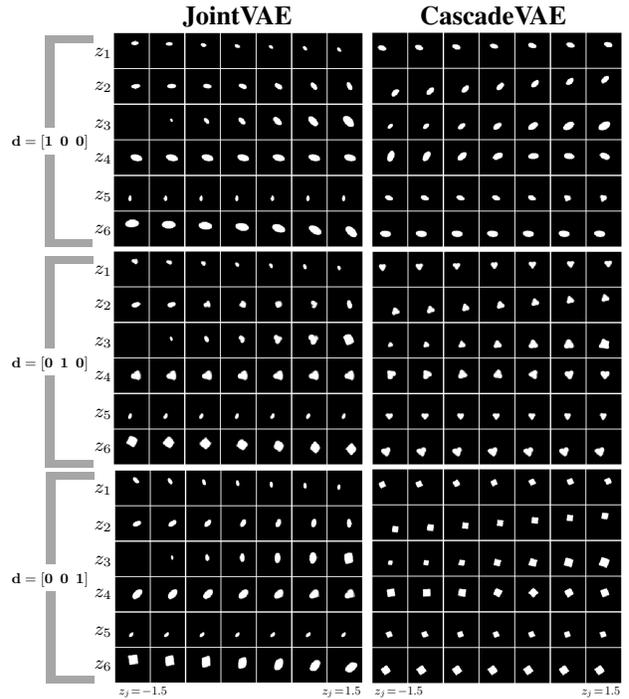
\begin{figure}[htbp]
\begin{tikzpicture}
\node (img) {\includegraphics[width=0.97\linewidth]{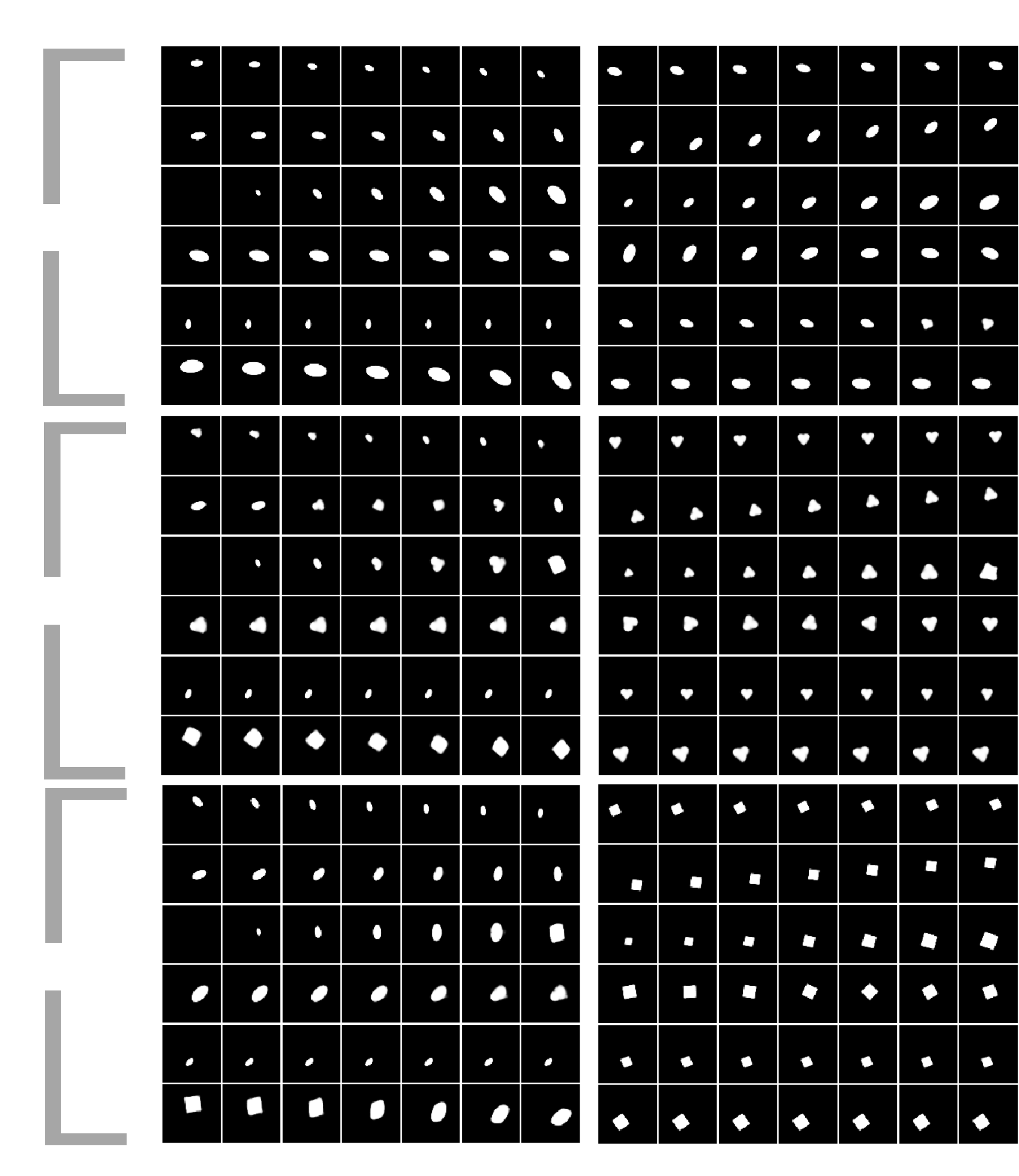}};
\node[text=black, below] at ([xshift=-1.1cm,yshift=0cm]img.north){\textbf{JointVAE}};
\node[text=black, below] at ([xshift=2.3cm,yshift=0cm]img.north){\textbf{CascadeVAE}};
\node[text=black, font=\footnotesize, right] at ([xshift=-2.8cm,yshift=0.25cm]img.south){\scalebox{0.5}{$z_j\!=\!-1.5$}};
\node[text=black, font=\footnotesize, right] at ([xshift=-0.1cm,yshift=0.25cm]img.south){\scalebox{0.5}{$z_j\!=\!1.5$}};
\node[text=black, font=\footnotesize, right] at ([xshift=0.6cm,yshift=0.25cm]img.south){\scalebox{0.5}{$z_j\!=\!-1.5$}};
\node[text=black, font=\footnotesize, right] at ([xshift=3.35cm,yshift=0.25cm]img.south){\scalebox{0.5}{$z_j\!=\!1.5$}};

\node[text=black, font=\footnotesize, right] at ([xshift=-0.1cm,yshift=2.82cm]img.west){\scalebox{0.6}{$\mathbf{d=[1~~0~~0]}$}};
\node[text=black, font=\footnotesize, right] at ([xshift=-0.1cm,yshift=-0.13cm]img.west){\scalebox{0.6}{$\mathbf{d=[0~~1~~0]}$}};
\node[text=black, font=\footnotesize, right] at ([xshift=-0.1cm,yshift=-3.0cm]img.west){\scalebox{0.6}{$\mathbf{d=[0~~0~~1]}$}};

\node[text=black, font=\footnotesize, right] at ([xshift=1.0cm,yshift=4cm]img.west){\scalebox{0.8}{$z_1$}};
\node[text=black, font=\footnotesize, right] at ([xshift=1.0cm,yshift=3.55cm]img.west){\scalebox{0.8}{$z_2$}};
\node[text=black, font=\footnotesize, right] at ([xshift=1.0cm,yshift=3.05cm]img.west){\scalebox{0.8}{$z_3$}};
\node[text=black, font=\footnotesize, right] at ([xshift=1.0cm,yshift=2.6cm]img.west){\scalebox{0.8}{$z_4$}};
\node[text=black, font=\footnotesize, right] at ([xshift=1.0cm,yshift=2.1cm]img.west){\scalebox{0.8}{$z_5$}};
\node[text=black, font=\footnotesize, right] at ([xshift=1.0cm,yshift=1.65cm]img.west){\scalebox{0.8}{$z_6$}};

\node[text=black, font=\footnotesize, right] at ([xshift=1.0cm,yshift=1.1cm]img.west){\scalebox{0.8}{$z_1$}};
\node[text=black, font=\footnotesize, right] at ([xshift=1.0cm,yshift=0.65cm]img.west){\scalebox{0.8}{$z_2$}};
\node[text=black, font=\footnotesize, right] at ([xshift=1.0cm,yshift=0.15cm]img.west){\scalebox{0.8}{$z_3$}};
\node[text=black, font=\footnotesize, right] at ([xshift=1.0cm,yshift=-0.3cm]img.west){\scalebox{0.8}{$z_4$}};
\node[text=black, font=\footnotesize, right] at ([xshift=1.0cm,yshift=-0.8cm]img.west){\scalebox{0.8}{$z_5$}};
\node[text=black, font=\footnotesize, right] at ([xshift=1.0cm,yshift=-1.25cm]img.west){\scalebox{0.8}{$z_6$}};

\node[text=black, font=\footnotesize, right] at ([xshift=1.0cm,yshift=-1.8cm]img.west){\scalebox{0.8}{$z_1$}};
\node[text=black, font=\footnotesize, right] at ([xshift=1.0cm,yshift=-2.3cm]img.west){\scalebox{0.8}{$z_2$}};
\node[text=black, font=\footnotesize, right] at ([xshift=1.0cm,yshift=-2.76cm]img.west){\scalebox{0.8}{$z_3$}};
\node[text=black, font=\footnotesize, right] at ([xshift=1.0cm,yshift=-3.2cm]img.west){\scalebox{0.8}{$z_4$}};
\node[text=black, font=\footnotesize, right] at ([xshift=1.0cm,yshift=-3.7cm]img.west){\scalebox{0.8}{$z_5$}};
\node[text=black, font=\footnotesize, right] at ([xshift=1.0cm,yshift=-4.15cm]img.west){\scalebox{0.8}{$z_6$}};
\end{tikzpicture}
\caption{DSprites latent space traversal results. (Left) Best run of Joint VAE with disentanglement score (91.75). (Right) Best run of CascadeVAE with disentanglement score (99.50)}
\label{fig:dsprites_latent_traversal_joint_ours}
\end{figure}

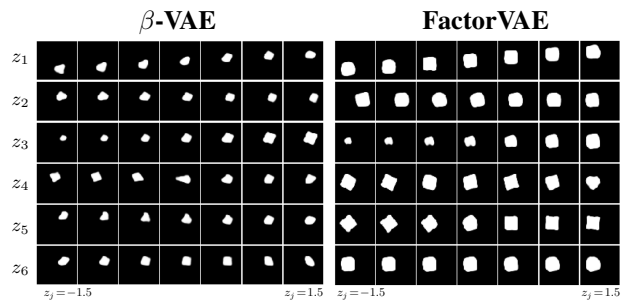
\begin{figure}[htbp]
\begin{tikzpicture}
\node (img) {\includegraphics[width=0.97\linewidth]{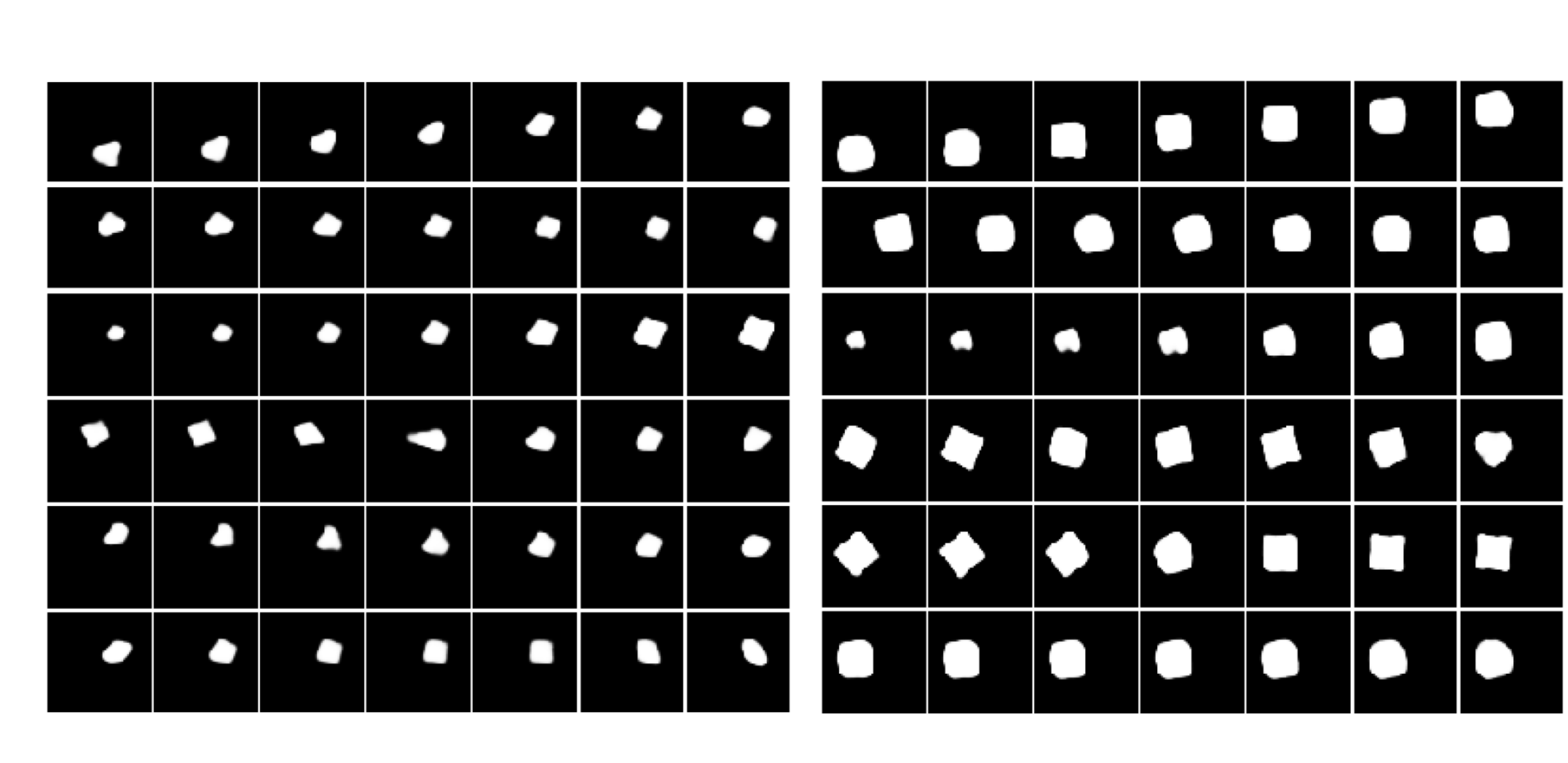}};
\node[text=black, below] at ([xshift=-1.9cm,yshift=0cm]img.north){\textbf{$\beta$-VAE}};
\node[text=black, below] at ([xshift=2.2cm,yshift=0cm]img.north){\textbf{FactorVAE}};
\node[text=black, font=\footnotesize, right] at ([xshift=-3.8cm,yshift=0.25cm]img.south){\scalebox{0.5}{\textbf{$z_j\!=\!-1.5$}}};
\node[text=black, font=\footnotesize, right] at ([xshift=-0.6cm,yshift=0.25cm]img.south){\scalebox{0.5}{\textbf{$z_j\!=\!1.5$}}};
\node[text=black, font=\footnotesize, right] at ([xshift=0.1cm,yshift=0.25cm]img.south){\scalebox{0.5}{\textbf{$z_j\!=\!-1.5$}}};
\node[text=black, font=\footnotesize, right] at ([xshift=3.35cm,yshift=0.25cm]img.south){\scalebox{0.5}{\textbf{$z_j\!=\!1.5$}}};

\node[text=black, font=\footnotesize, right] at ([xshift=-0.1cm,yshift=1.3cm]img.west){\scalebox{0.8}{\textbf{$z_1$}}};
\node[text=black, font=\footnotesize, right] at ([xshift=-0.1cm,yshift=0.73cm]img.west){\scalebox{0.8}{\textbf{$z_2$}}};
\node[text=black, font=\footnotesize, right] at ([xshift=-0.1cm,yshift=0.17cm]img.west){\scalebox{0.8}{\textbf{$z_3$}}};
\node[text=black, font=\footnotesize, right] at ([xshift=-0.1cm,yshift=-0.37cm]img.west){\scalebox{0.8}{\textbf{$z_4$}}};
\node[text=black, font=\footnotesize, right] at ([xshift=-0.1cm,yshift=-0.95cm]img.west){\scalebox{0.8}{\textbf{$z_5$}}};
\node[text=black, font=\footnotesize, right] at ([xshift=-0.1cm,yshift=-1.5cm]img.west){\scalebox{0.8}{\textbf{$z_6$}}};
\end{tikzpicture}
\caption{DSprites latent space traversal results. (Left) $\beta$-VAE, (Right) FactorVAE}
\label{fig:dsprites_latent_traversal_beta_factor}
\end{figure}

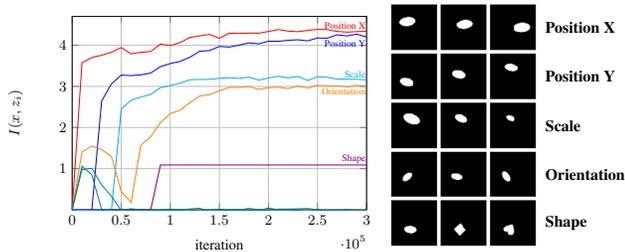
\begin{figure}[h]
\begin{subfigure}[t]{0.60\linewidth}
\centering
\begin{adjustbox}{max width=\columnwidth}
\begin{tikzpicture}
\begin{axis}[no markers, width=8.0cm,height=5.8cm, tick label style={font=\small}, ytick={1,2,3,4}, grid=major, label style={font=\footnotesize}, xlabel={iteration}, ylabel={$I(x,z_i)$}, xmin=0.0, xmax=300000, ymin=0.0, ymax=4.7]
\node[text=violet, font=\tiny, left] at (305,130) {Shape};
\node[text=red, font=\tiny, left] at (305,455) {Position X};
\node[text=blue, font=\tiny, left] at (305,410) {Position Y};
\node[text=cyan, font=\tiny, left] at (305,335) {Scale};
\node[text=orange, font=\tiny, left] at (305,295) {Orientation};

\addplot+[violet, solid] table [x=iter, y=value, col sep=comma] {csv/dsprites_mi_d.csv};
\addplot+[red, solid] table [x=iter, y=value, col sep=comma] {csv/dsprites_mi0.csv};
\addplot+[blue, solid] table [x=iter, y=value, col sep=comma] {csv/dsprites_mi1.csv};
\addplot+[cyan, solid] table [x=iter, y=value, col sep=comma] {csv/dsprites_mi2.csv};
\addplot+[orange, solid] table [x=iter, y=value, col sep=comma] {csv/dsprites_mi3.csv};
\addplot+[teal, solid] table [x=iter, y=value, col sep=comma] {csv/dsprites_mi4.csv};
\addplot+[teal, solid] table [x=iter, y=value, col sep=comma] {csv/dsprites_mi5.csv};
\end{axis}
\end{tikzpicture}
\end{adjustbox}
\end{subfigure}
\begin{subfigure}[t]{0.37\linewidth}
\begin{tikzpicture}
\node (img) {\includegraphics[width=\linewidth]{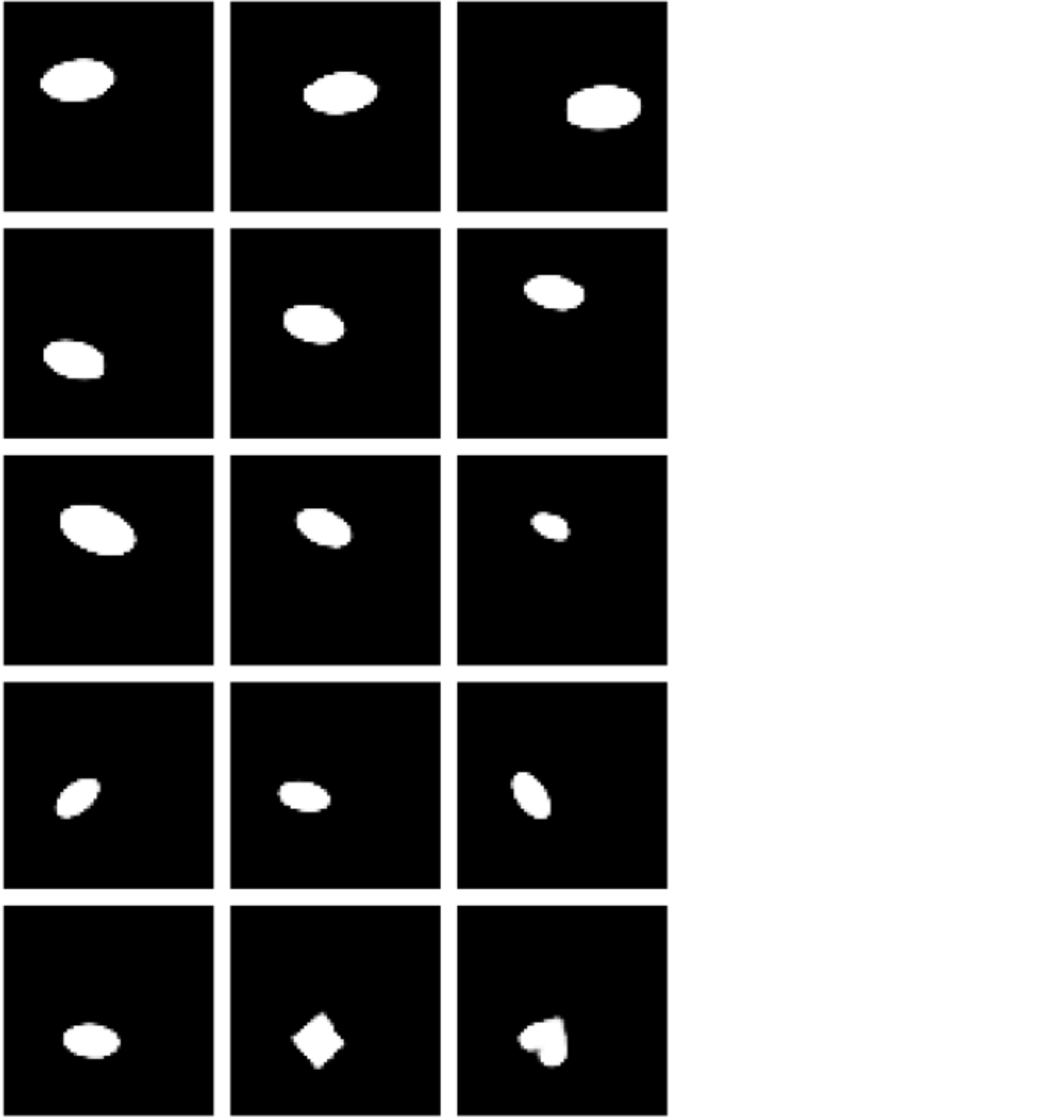}};
\node[text=black, font=\tiny, right] at ([xshift=-1.25cm,yshift=1.3cm]img.east){\textbf{Position X}};
\node[text=black, font=\tiny, right] at ([xshift=-1.25cm,yshift=0.65cm]img.east){\textbf{Position Y}};
\node[text=black, font=\tiny, right] at ([xshift=-1.25cm,yshift=0.0cm]img.east){\textbf{Scale}};
\node[text=black, font=\tiny, right] at ([xshift=-1.25cm,yshift=-0.65cm]img.east){\textbf{Orientation}};
\node[text=black, font=\tiny, right] at ([xshift=-1.25cm,yshift=-1.3cm]img.east){\textbf{Shape}};
\end{tikzpicture}
\end{subfigure}\hspace{0.005\linewidth}
\caption{(Left) Increase of mutual information between image and each dimension $I(x,z_j)$ or $I(x,d)$ during training on Dsprites dataset.
    (Right) Each row respresents latent traversal across each dimension sorted by $I(x,z_j)$. Traversal range is $(-1.2, 1.2)$.}
\label{fig:dsprites_plot}
\end{figure}

\subsection{Experiments on MNIST}
MNIST has $60,000$ images of size $28\times 28$. We split $50,000$ as training images and $10,000$ as test images. MNIST dataset has 10 discrete digit categories. For quantitative comparison, \Cref{tab:mnist_cls} compares the classification accuracy computed from the discrete variables from inference against JointVAE. The results show that CascadeVAE outperforms JointVAE by $15\%$ classification accuracy. We fixed $S=10$ following the experiment protocol in JointVAE for a fair comparison.

For qualitative results, \Cref{fig:mnist_continuous_latent_traversal} and \Cref{fig:mnist_discrete_latent_traversal} show the continuous and discrete latent traversal results, respectively. The results show smooth transitions in angle, width, stroke, and thickness respectively for continuous traversal. The discrete traversal shows that it captures the categorical information of MNIST. \Cref{fig:mnist_plot} (left) shows the iteration versus mutual information of each latent dimension $I(x;z_i)$ plot when CascadeVAE is trained with 10 continuous and 1 discrete variables. \Cref{fig:mnist_plot} (right) shows the traversal results for each variables.

\begin{table}[htbp]
\centering
\fontsize{9pt}{9.5pt}\selectfont
\begin{tabular}{rc rr}
\addlinespace[-\aboverulesep]
\toprule
\multicolumn{1}{c}{Method}&m& \multicolumn{1}{c}{Mean (std)}&\multicolumn{1}{c}{Best}\\
\toprule
JointVAE & 10&68.57 (9.19) &82.30\\
        &  4&78.33 (7.18) &92.81\\
\midrule
CascadeVAE &10&81.41 (9.54)& \textbf{97.31}\\
           & 4&\textbf{84.19 (5.02)}& 96.39\\
\bottomrule
\end{tabular}
\caption{Unsupervised classification results on MNIST, $S=10$.  Unsupervised classification accuray for random chance is $10.00$.}
\label{tab:mnist_cls}
\end{table}

\begin{figure}[bp]
\centering
\begin{subfigure}[t]{0.48\linewidth}
\centering
\includegraphics[width=1.0\linewidth]{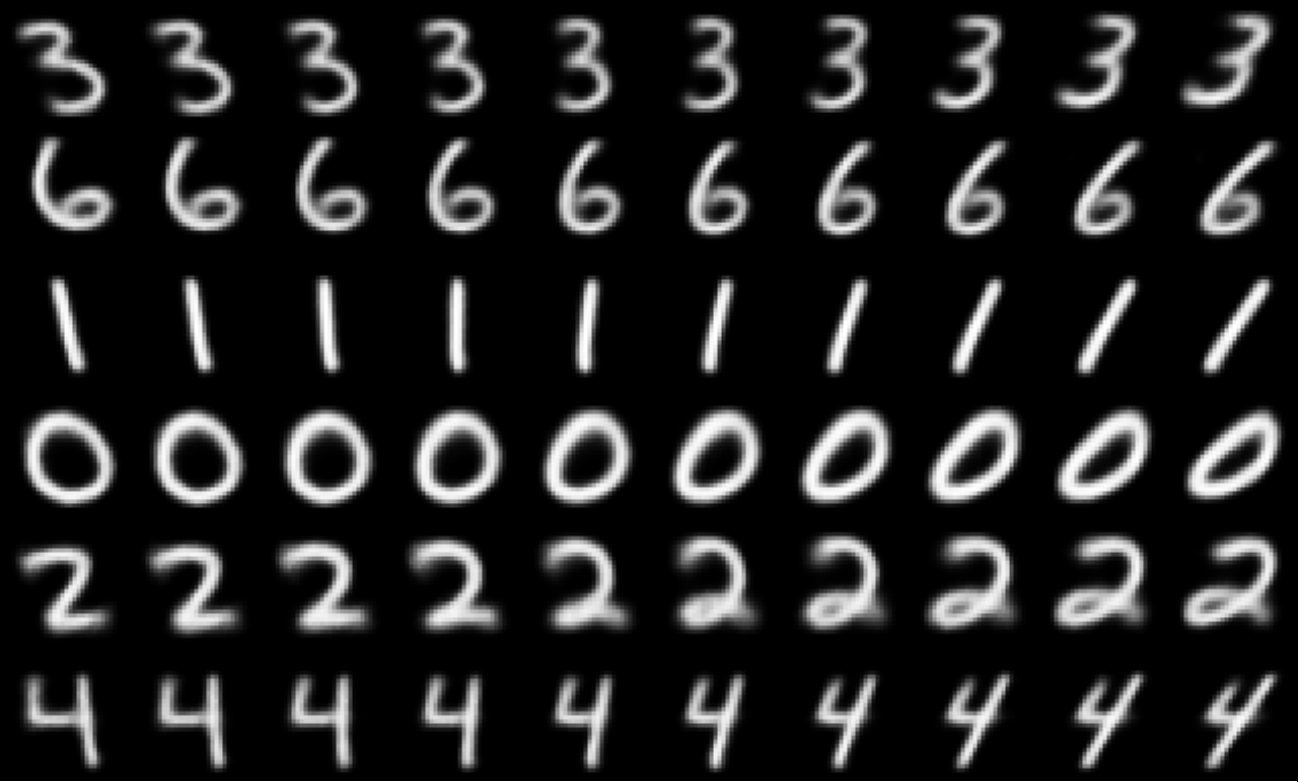}
\caption{Angle}
\end{subfigure}\hspace{0.005\linewidth}
\begin{subfigure}[t]{0.48\linewidth}
\centering
\includegraphics[width=1.0\linewidth]{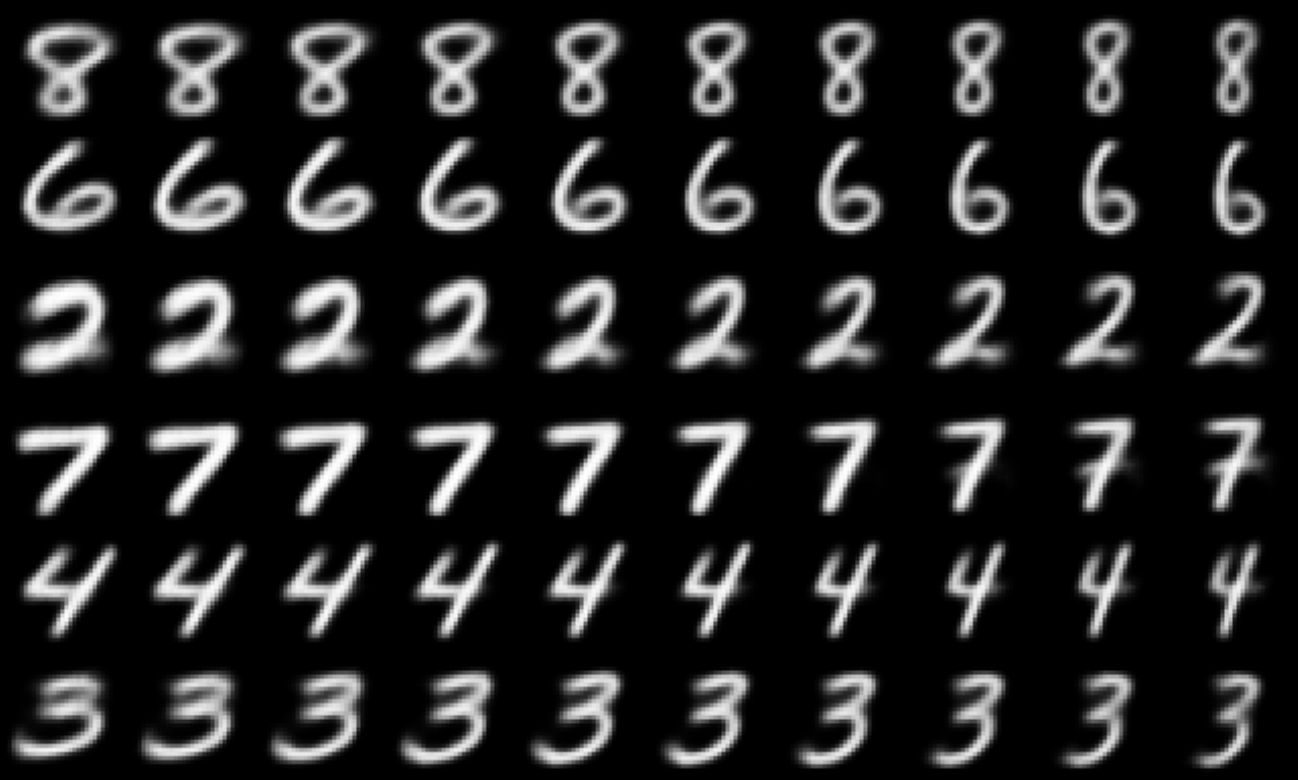}
\caption{Width}
\end{subfigure}\hspace{0.005\linewidth}
\begin{subfigure}[t]{0.48\linewidth}
\centering
\includegraphics[width=1.0\linewidth]{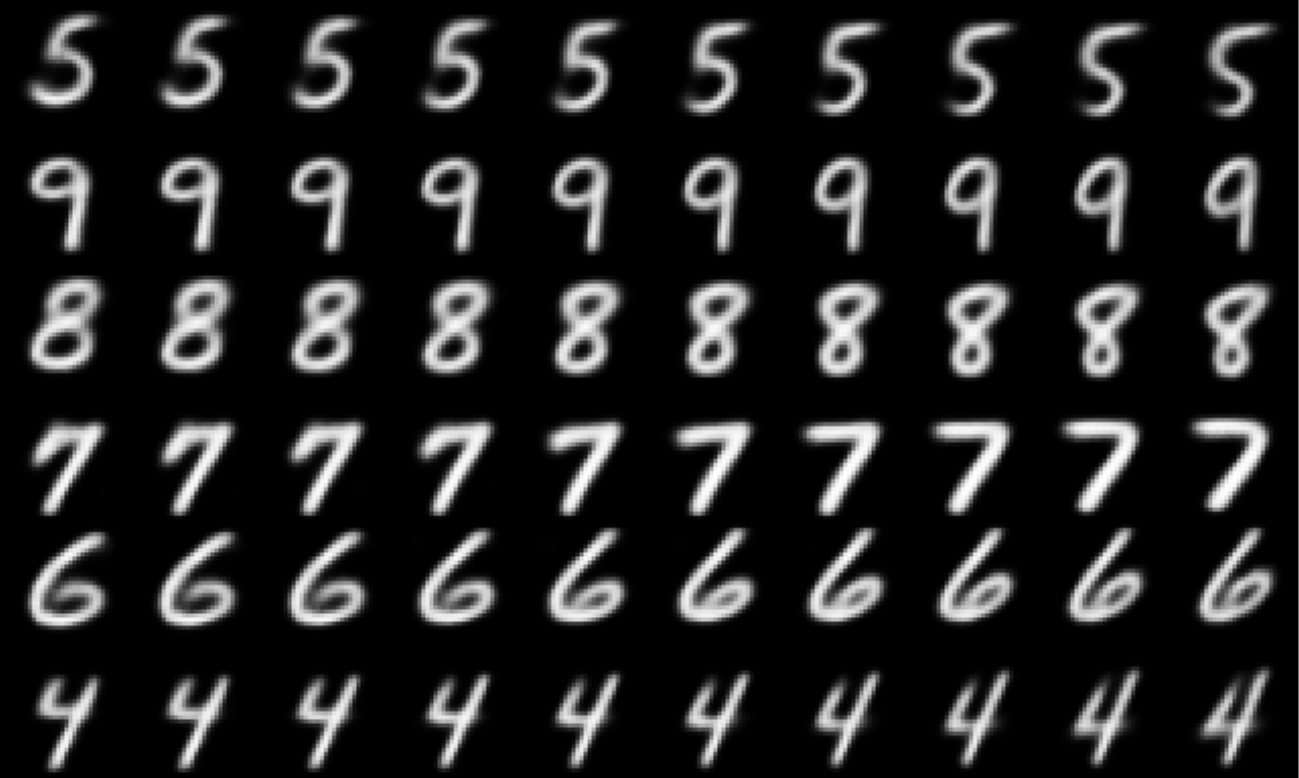}
\caption{Stroke}
\end{subfigure}\hspace{0.005\linewidth}
\begin{subfigure}[t]{0.48\linewidth}
\centering
\includegraphics[width=1.0\linewidth]{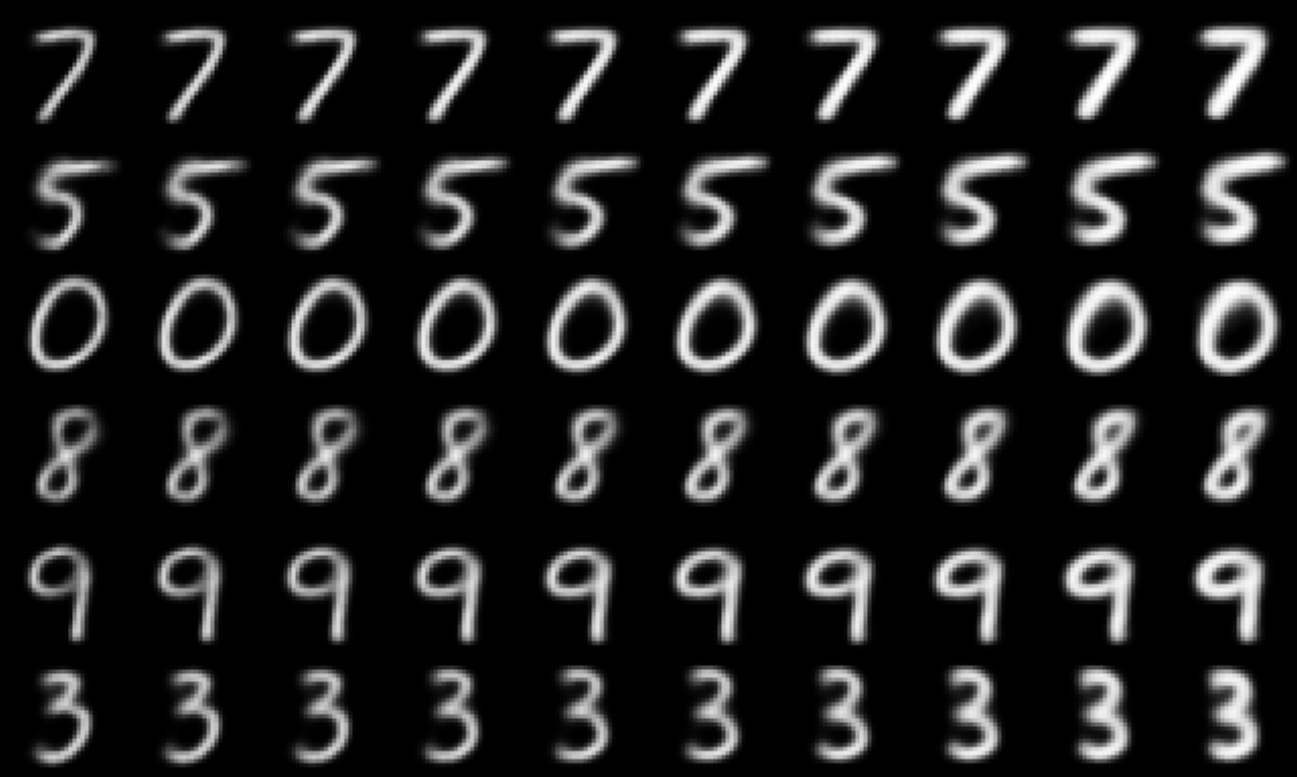}
\caption{Thickness}
\end{subfigure}
\caption{
    Latent traversals on MNIST. Images in a row has the same latent variables except the traversed variable.}
\label{fig:mnist_continuous_latent_traversal}
\end{figure}

\begin{figure}[htbp]
\includegraphics[width=1.0\linewidth]{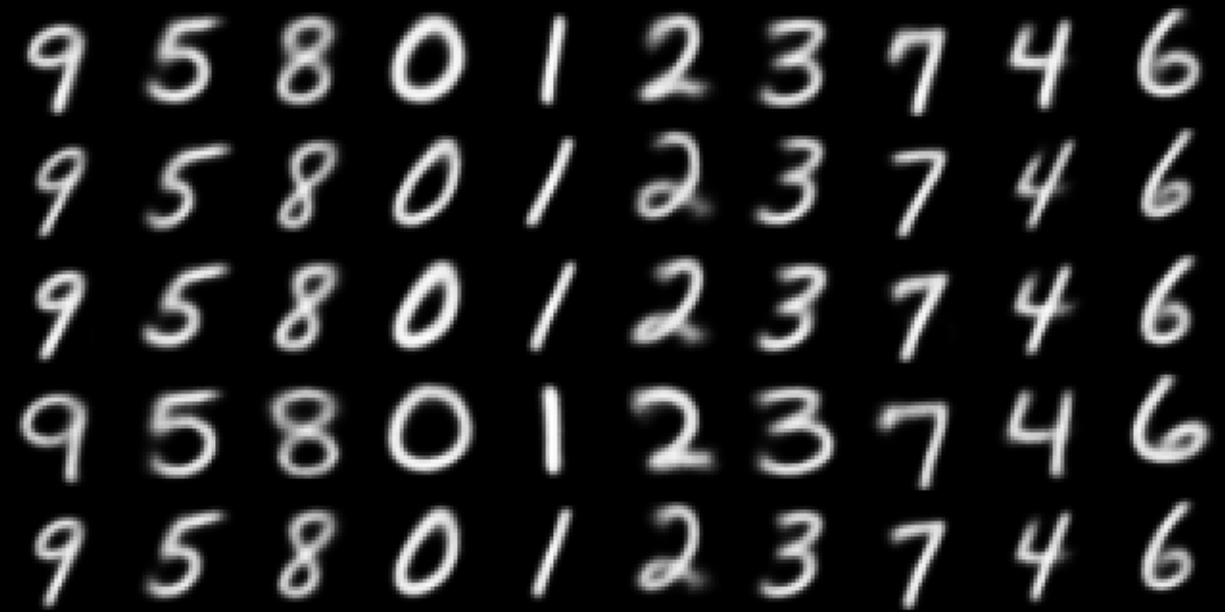}
\caption{MNIST discrete latent space traversal}
\label{fig:mnist_discrete_latent_traversal}
\end{figure}

\begin{figure}[htbp]
\centering
\begin{subfigure}[t]{0.48\linewidth}
\centering
\includegraphics[width=1.0\linewidth]{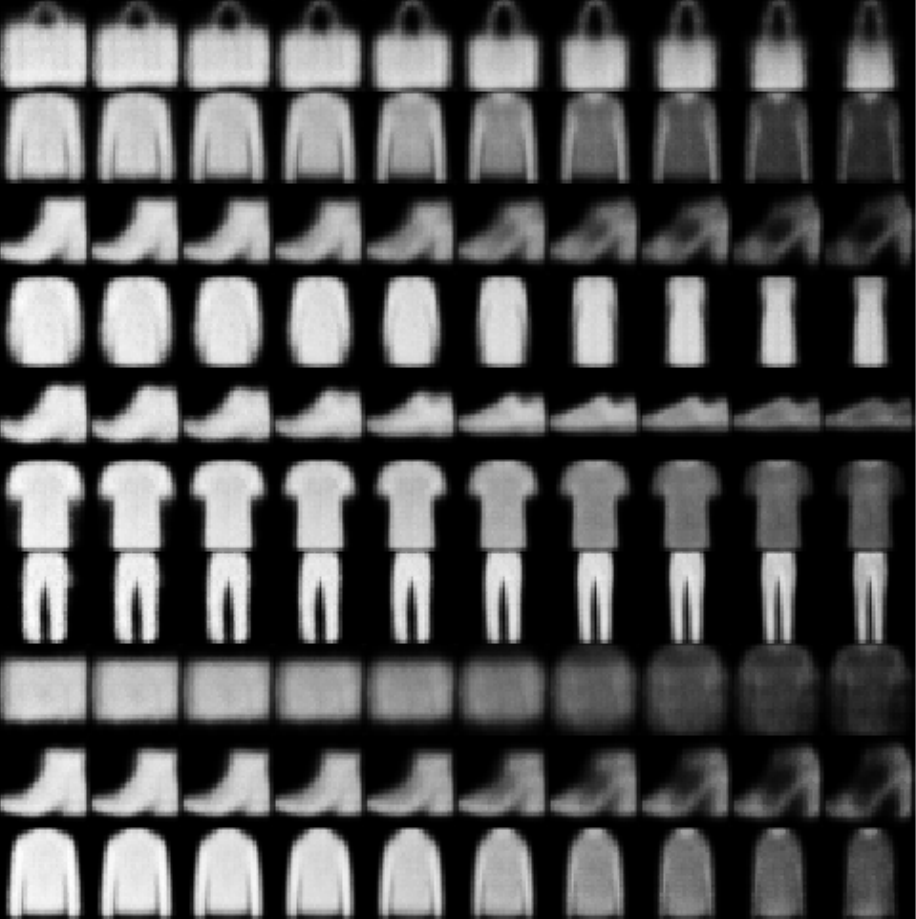}
\caption{Continous}
\end{subfigure}\hspace{0.005\linewidth}
\begin{subfigure}[t]{0.48\linewidth}
\centering
\includegraphics[width=1.0\linewidth]{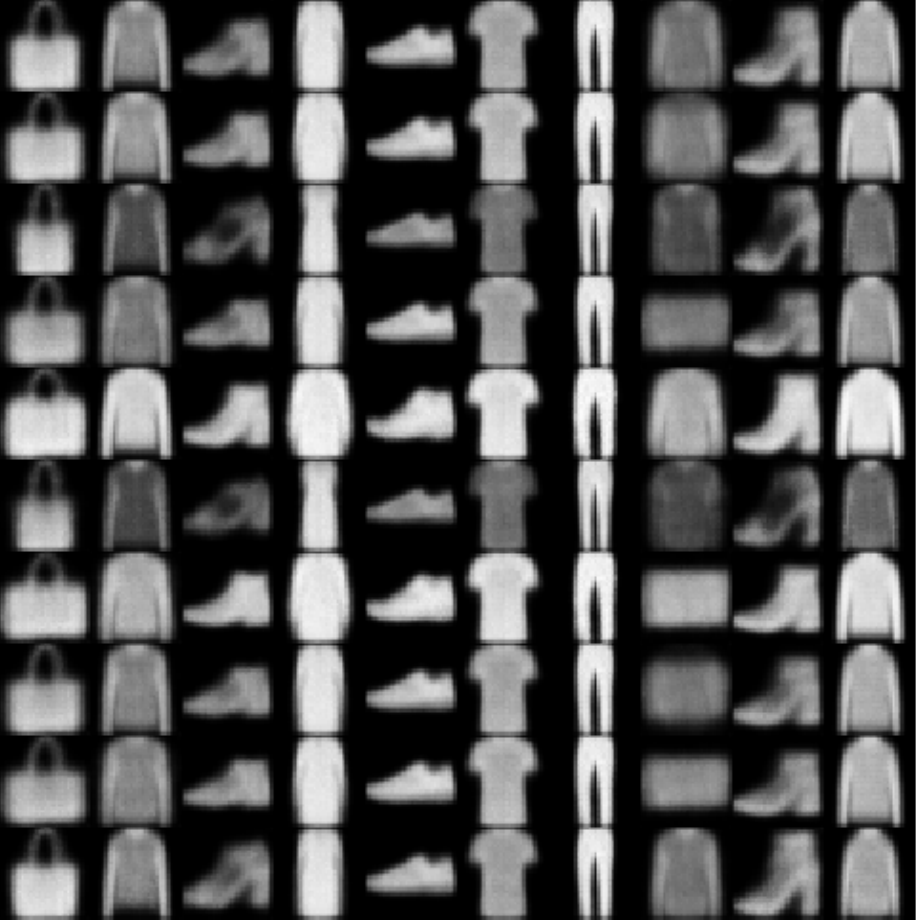}
\caption{Discrete}
\end{subfigure}\hspace{0.005\linewidth}
\caption{Latent traversals of the most informative continuous dimension (left) and the discrete dimension (right) on FashionMNIST. Traversal range in continuous dimension is (-2.0, 2.0).}
\label{fig:fmnist_latent_traversal}
\end{figure}

\begin{figure}[htbp]
\begin{subfigure}[t]{0.59\linewidth}
\centering
\begin{adjustbox}{max width=\columnwidth}
\begin{tikzpicture}
\begin{axis}[no markers, tick label style={font=\small}, grid=major, label style={font=\small}, xlabel={iteration}, ylabel={$I(x,z_i)$}, xmin=0.0, xmax=120000, ymin=0.0, ymax=2.5]
\node[text=red, font=\footnotesize, left] at (120,240) {Digit};
\node[text=brown, font=\footnotesize, left] at (120,215) {Slant};
\node[text=violet, font=\footnotesize, left] at (120,187) {Width};
\node[text=orange, font=\footnotesize, left] at (120,157) {Stroke1};
\node[text=green, font=\footnotesize, left] at (120,133) {Thickness1};
\node[text=teal, font=\footnotesize, left] at (120,108) {Stroke2};
\node[text=cyan, font=\footnotesize, left] at (120,83) {Thickness2};
\node[text=blue, font=\footnotesize, left] at (120,55) {Stroke3};

\addplot+[red, solid] table [x=iter, y=value, col sep=comma] {csv/mnist_mi_d.csv};
\addplot+[brown, solid] table [x=iter, y=value, col sep=comma] {csv/mnist_mi0.csv};
\addplot+[violet, solid] table [x=iter, y=value, col sep=comma] {csv/mnist_mi1.csv};
\addplot+[orange, solid] table [x=iter, y=value, col sep=comma] {csv/mnist_mi2.csv};
\addplot+[green, solid] table [x=iter, y=value, col sep=comma] {csv/mnist_mi3.csv};
\addplot+[teal, solid] table [x=iter, y=value, col sep=comma] {csv/mnist_mi4.csv};
\addplot+[cyan, solid] table [x=iter, y=value, col sep=comma] {csv/mnist_mi5.csv};
\addplot+[blue, solid] table [x=iter, y=value, col sep=comma] {csv/mnist_mi6.csv};
\addplot+[magenta, solid] table [x=iter, y=value, col sep=comma] {csv/mnist_mi7.csv};
\addplot+[black, solid] table [x=iter, y=value, col sep=comma] {csv/mnist_mi8.csv};
\addplot+[gray, solid] table [x=iter, y=value, col sep=comma] {csv/mnist_mi9.csv};
\end{axis}
\end{tikzpicture}
\end{adjustbox}
\end{subfigure}
\begin{subfigure}[t]{0.37\linewidth}
\begin{tikzpicture}
\node (img) {\includegraphics[width=\linewidth]{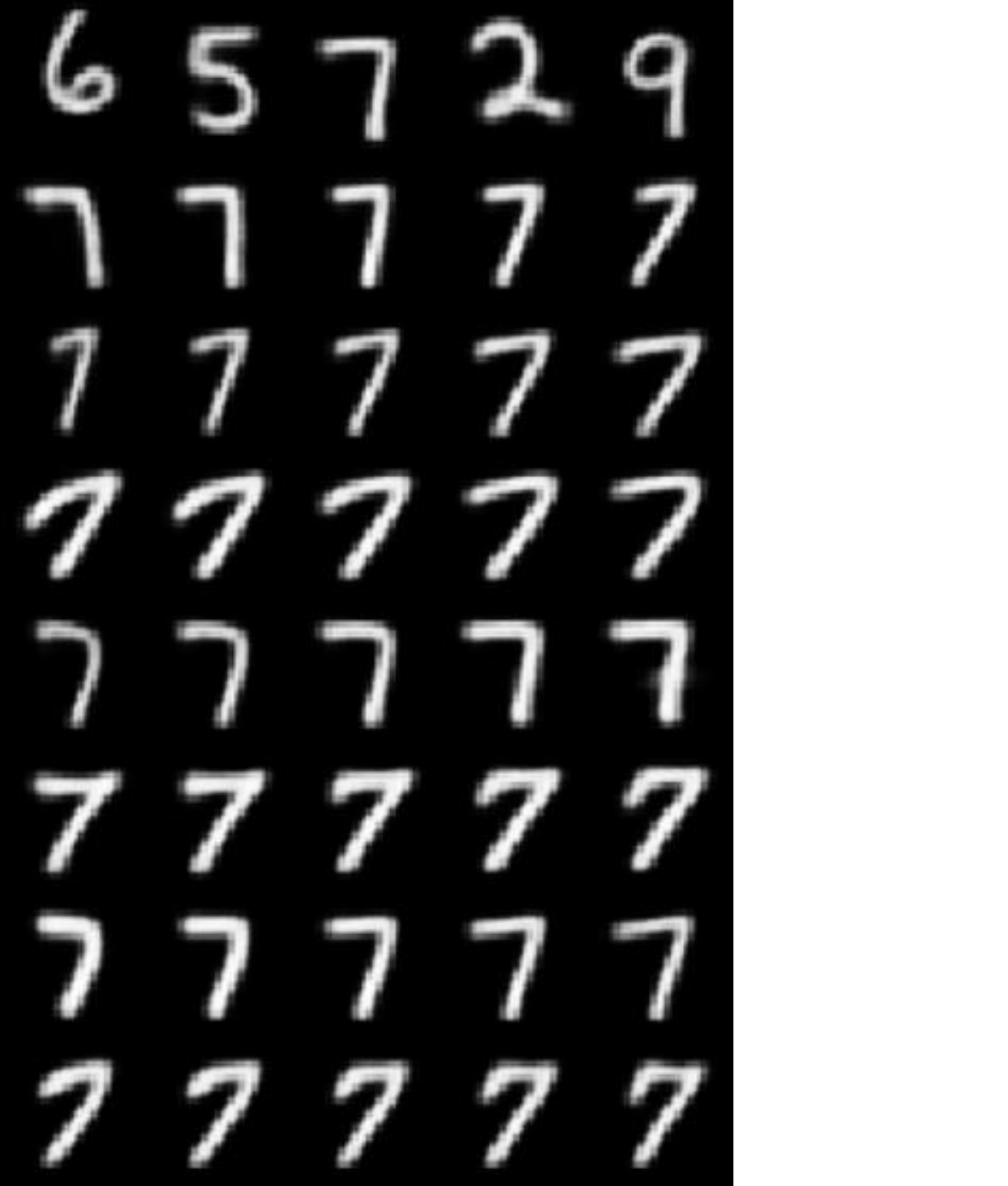}};
\node[text=black, font=\tiny, right] at ([xshift=-1cm,yshift=1.55cm]img.east){\scalebox{0.8}{\textbf{Digit}}};
\node[text=black, font=\tiny, right] at ([xshift=-1cm,yshift=1.1cm]img.east){\scalebox{0.8}{\textbf{Slant}}};
\node[text=black, font=\tiny, right] at ([xshift=-1cm,yshift=0.6cm]img.east){\scalebox{0.8}{\textbf{Width}}};
\node[text=black, font=\tiny, right] at ([xshift=-1cm,yshift=0.2cm]img.east){\scalebox{0.8}{\textbf{Stroke1}}};
\node[text=black, font=\tiny, right] at ([xshift=-1cm,yshift=-0.25cm]img.east){\scalebox{0.8}{\textbf{Thickness1}}};
\node[text=black, font=\tiny, right] at ([xshift=-1cm,yshift=-0.7cm]img.east){\scalebox{0.8}{\textbf{Stroke2}}};
\node[text=black, font=\tiny, right] at ([xshift=-1cm,yshift=-1.15cm]img.east){\scalebox{0.8}{\textbf{Thickness2}}};
\node[text=black, font=\tiny, right] at ([xshift=-1cm,yshift=-1.6cm]img.east){\scalebox{0.8}{\textbf{Stroke3}}};
\end{tikzpicture}
\end{subfigure}\hspace{0.005\linewidth}
\caption{(Left) Increase of mutual information between image and each continuous $I(x,z_j)$ and discrete factors $I(x,d)$ during training on MNIST dataset.
    (Right) Each row respresents latent traversal across each dimension sorted by $I(x,z_j)$. Traversal range is $(-2.0, 2.0)$.}
\label{fig:mnist_plot}
\end{figure}
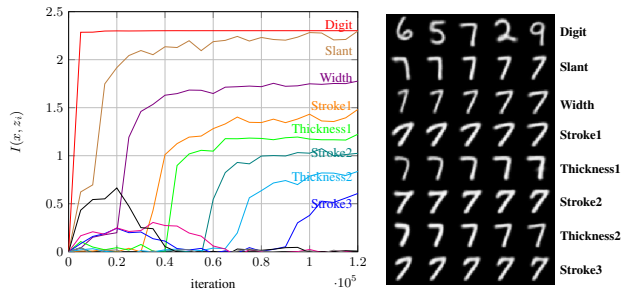

\begin{figure}[htbp]
\begin{tikzpicture}
\node (img) {\includegraphics[width=0.97\linewidth]{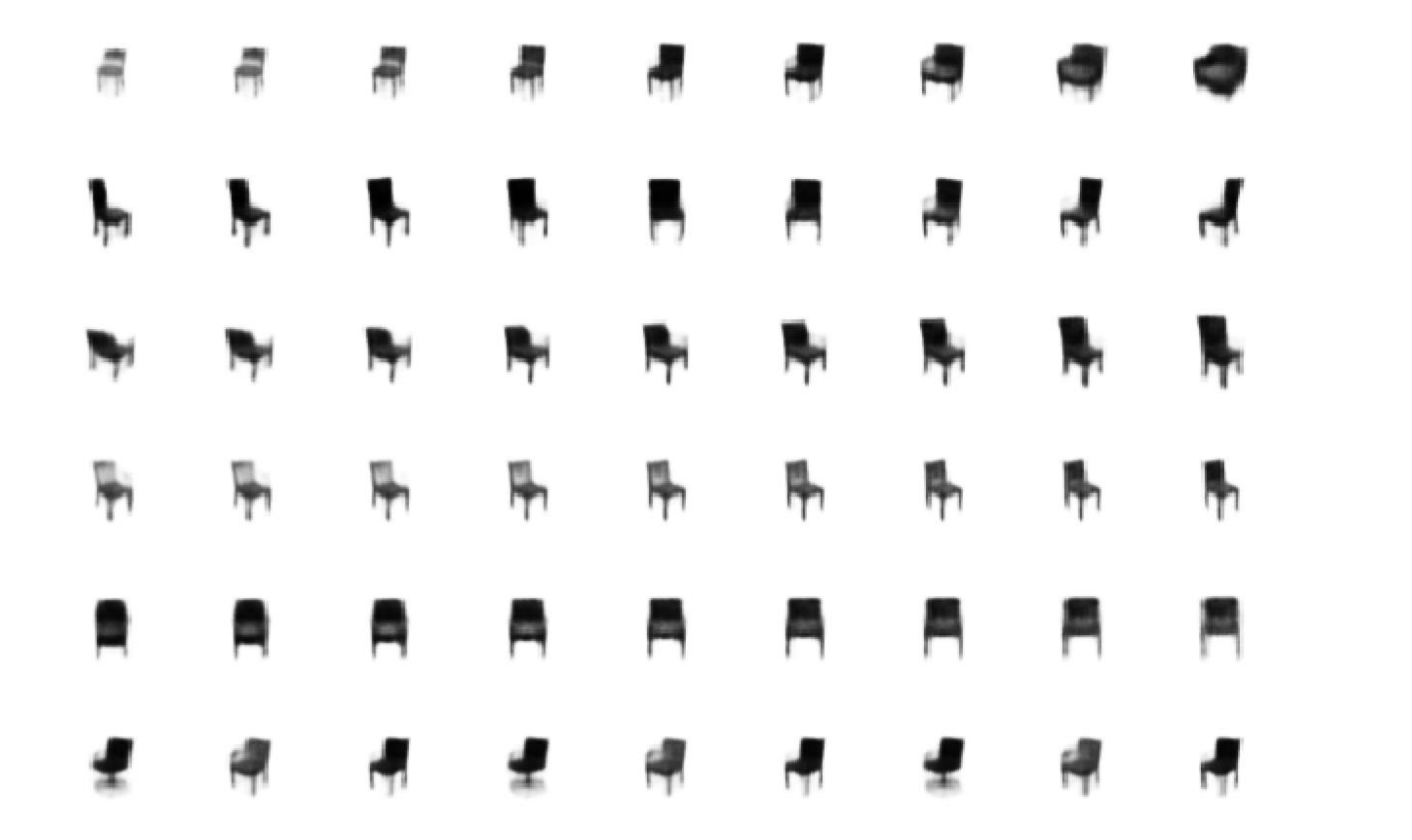}};
\node[text=black, font=\footnotesize, right] at ([xshift=0.2cm,yshift=2.0cm]img.west){\scalebox{0.75}{\textbf{$z_1$}}};
\node[text=black, font=\footnotesize, right] at ([xshift=0.2cm,yshift=1.2cm]img.west){\scalebox{0.75}{\textbf{$z_2$}}};
\node[text=black, font=\footnotesize, right] at ([xshift=0.2cm,yshift=0.4cm]img.west){\scalebox{0.75}{\textbf{$z_3$}}};
\node[text=black, font=\footnotesize, right] at ([xshift=0.2cm,yshift=-0.35cm]img.west){\scalebox{0.75}{\textbf{$z_4$}}};
\node[text=black, font=\footnotesize, right] at ([xshift=0.2cm,yshift=-1.15cm]img.west){\scalebox{0.75}{\textbf{$z_5$}}};
\node[text=black, font=\footnotesize, right] at ([xshift=0.2cm,yshift=-1.9cm]img.west){\scalebox{0.75}{\textbf{$d$}}};
\node[text=black, font=\footnotesize, right] at ([xshift=-1cm,yshift=2.0cm]img.east){\scalebox{0.55}{\textbf{Size}}};
\node[text=black, font=\footnotesize, right] at ([xshift=-1cm,yshift=1.2cm]img.east){\scalebox{0.55}{\textbf{Azimuth}}};
\node[text=black, font=\footnotesize, right] at ([xshift=-1cm,yshift=0.4cm]img.east){\scalebox{0.55}{\textbf{Backrest}}};
\node[text=black, font=\footnotesize, right] at ([xshift=-1cm,yshift=-0.35cm]img.east){\scalebox{0.55}{\textbf{Material}}};
\node[text=black, font=\footnotesize, right] at ([xshift=-1cm,yshift=-1.15cm]img.east){\scalebox{0.55}{\textbf{Leg length}}};
\node[text=black, font=\footnotesize, right] at ([xshift=-1cm,yshift=-1.9cm]img.east){\scalebox{0.55}{\textbf{Leg and arm}}};

\end{tikzpicture}
\caption{Latent space traversal on chairs dataset. The last row shows the latent traversal of the discrete factor of dimension 3 with the period of 3.}
\label{fig:chairs_latent_traversal}
\end{figure}
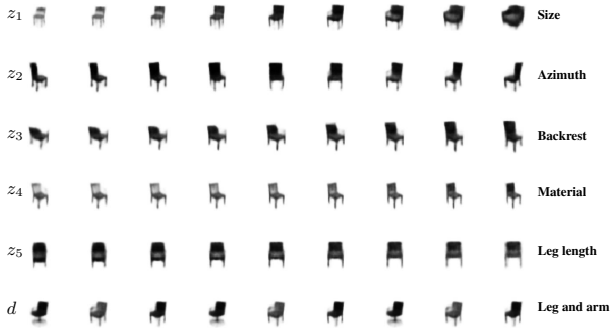

\begin{figure}[htbp]
\begin{adjustbox}{max width=\columnwidth}
\begin{tikzpicture}
\begin{axis}[no markers, width=8.0cm,height=5.8cm, tick label style={font=\small}, grid=major, label style={font=\footnotesize}, xlabel={Iteration}, ylabel={$I(x,z_i)$}, xmin=0.0, xmax=150000, ymin=0.0, ymax=4.0]
\node[text=violet, font=\tiny, left] at (150,120) {Leg and arm};
\node[text=red, font=\tiny, left] at (150,340) {Size};
\node[text=blue, font=\tiny, left] at (150,380) {Azimuth};
\node[text=cyan, font=\tiny, left] at (150,290) {Backrest};
\node[text=orange, font=\tiny, left] at (150, 210) {Material};
\node[text=teal, font=\tiny, left] at (150,237) {Leg length};

\addplot+[violet, solid] table [x=iter, y=value, col sep=comma] {csv/chairs_mi_d.csv};
\addplot+[red, solid] table [x=iter, y=value, col sep=comma] {csv/chairs_mi0.csv};
\addplot+[blue, solid] table [x=iter, y=value, col sep=comma] {csv/chairs_mi1.csv};
\addplot+[cyan, solid] table [x=iter, y=value, col sep=comma] {csv/chairs_mi2.csv};
\addplot+[orange, solid] table [x=iter, y=value, col sep=comma] {csv/chairs_mi3.csv};
\addplot+[teal, solid] table [x=iter, y=value, col sep=comma] {csv/chairs_mi4.csv};
\end{axis}
\end{tikzpicture}
\end{adjustbox}
\caption{Increase of $I(x,z_i)$ during training the chairs dataset}
\label{fig:chairs_plot}
\end{figure}
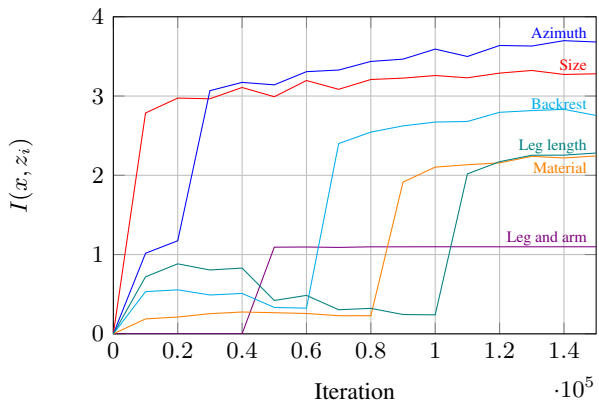

\begin{table}[htbp]
\fontsize{9pt}{9.5pt}\selectfont
\centering
\begin{tabular}{rc ll}
\addlinespace[-\aboverulesep]
\toprule
\multicolumn{1}{c}{Method}&m (conti)&\multicolumn{1}{c}{Mean (std)}&\multicolumn{1}{c}{Best}\\
\toprule
JointVAE           &10&50.99 (2.5) &55.64\\
                  & 4&51.51 (4.42)&61.79\\
\midrule
CascadeVAE &10&56.01 (3.33)&63.01\\
           & 4&\textbf{57.72 (3.29)}&\textbf{63.55}\\
\bottomrule
\end{tabular}
\caption{Unsupervised classification results on FashionMNIST, $S=10$. Unsupervised classification accuray for random chance is $10.00$.}
\label{tab:fmnist_cls}
\end{table}

\subsection{Experiments on FashionMNIST}
FashionMNIST dataset has $60,000$ images of size $28\times 28$ with 10 categorial labels like MNIST. \Cref{tab:fmnist_cls} shows the unsupervised classification results from inference in comparison to JointVAE. The results show CascadeVAE consistently outperforms JointVAE. \Cref{fig:fmnist_latent_traversal} shows the latent traversal of the most informative continuous dimension and discrete dimension. 

\subsection{Chairs}
We preprocessed the chairs dataset \cite{chairs} and prepared $86,366$ images of size $64\times64$. Since the chairs dataset is unlabeled in contrast to MNIST, we can only evaluate the qualitative performance instead. \Cref{fig:chairs_latent_traversal} shows the latent traversal results on chairs. The traversal on discrete latent variables (the last row in the figure) shows that it captures to the categorial shape of different chairs. \Cref{fig:chairs_plot} shows the iteration versus the estimated mutual information $I(x,z_i)$ plot for each variables.

\section{Conclusion}\label{sec:conclusion}
We have developed CascadeVAE for jointly learning the discrete and continuous factors of data in a $\beta$-VAE framework. We first propose an efficient procedure for implicitly penalizing the total correlation by controlling the information flow on each variable. This allows us to penalize the total correlation without using extra discriminator networks or sampling procedures. Then, we show a method for jointly learning discrete and continuous latent variables in an alternating maximization framework where we alternate between finding the most likely discrete configurations based on the continuous latent variables, and updating the inference parameters based on the discrete variables. 

Our ablation study shows that information cascading and alternating maximization of discrete and continuous variables, provide complementary benefits and leads to the state of the art performance in 1) disentanglement score, and 2) classification accuracy score from the discrete inference network, compared to a number of recently proposed methods.

\section*{Acknowledgements}
This work was partially supported by Kakao, Kakao Brain and Basic Science Research Program through the National Research Foundation of Korea (NRF) (2017R1E1A1A01077431). Hyun Oh Song is the corresponding author.

\newpage
\bibliography{main}
\bibliographystyle{icml2019}

\newpage
\section*{Supplementary material}
\section*{A. Proofs}\label{sec:proofs}
\addtocounter{proposition}{0}
\setcounter{proposition}{0}

\subsection*{A1. Proof of Proposition 1}

\begin{proposition}
\label{prop1}
The mutual information between one dimension of a random variable and the rest can be factorized as
\[I(z_{1:i-1}; z_i) = TC(z_{1:i}) - TC(z_{1:i-1})\]
\end{proposition}
\begin{proof}
First recall the definition of total correlation, 
\[TC(z_{1:i}) = D_\text{KL}\left(p(z_{1:i}) \parallel \prod_{j=1}^i p(z_j)\right)\]
Then, we have
\begin{align*}
    TC(&z_{1:i}) - TC(z_{1:i-1})\\
    &= \int p(z_{1:i}) \log \frac{p(z_{1:i})}{\prod_{j=1}^i p(z_j)} dz_{1:i} \\
    &~~~~- \int p(z_{1:i-1}) \log \frac{p(z_{1:i-1})}{\prod_{j=1}^{i-1} p(z_{j})} dz_{1:i-1} \\
    &= \int p(z_{1:i}) \log \frac{p(z_{1:i})}{\prod_{j=1}^i p(z_j)} dz_{1:i} \\
    &~~~~- \int p(z_{1:i}) \log \frac{p(z_{1:i-1})}{\prod_{j=1}^{i-1} p(z_j)} dz_{1:i} \\
    &= \int p(z_{1:i}) \log \frac{p(z_{1:i})}{p(z_{1:i-1})p(z_i)} dz_{1:i} \\
    &= I(z_{1:i-1};z_i)
\end{align*}
\end{proof}

\subsection*{A2. Proof of Proposition 2}

\begin{proposition}
\label{prop2}
The mutual information between $x$ and partitions of $z = [z_1, z_2]$ can be factorized as,
\[I(x; [z_1, z_2]) = I(x;z_1) + I(x; z_2) - I(z_1; z_2)\]
\end{proposition}
\begin{proof}

Recall the conditional independence of the latent variables $p(z_1,z_2|x)=p(z_1|x)p(z_2|x)$,
\begin{align*}
    &I(x;[z_1,z_2])\\
    &= \int p(x,z_1,z_2) \log \frac{p(x,z_1,z_2)}{p(x)p(z_1,z_2)} dz_1dz_2 dx  \\
    &= \int p(x,z_1,z_2) \log \bigg(\frac{p(x,z_1,z_2)}{p(x)p(z_1,z_2)}\cdot \frac{p(x)p(z_1)}{p(x,z_1)} \\
    &\hspace{9.5em}\cdot\frac{p(x)p(z_2)}{p(x,z_2)}\cdot \frac{p(z_1,z_2)}{p(z_1)p(z_2)}\bigg) dz_1dz_2 dx  \\
    &~~~+ \int p(x,z_1,z_2) \log \frac{p(x,z_1)}{p(x)p(z_1)} dx dz_1 dz_2\\
    &~~~+ \int p(x,z_1,z_2) \log \frac{p(x,z_2)}{p(x)p(z_2)} dx dz_1 dz_2\\
    &~~~- \int p(x,z_1,z_2) \log \frac{p(z_1,z_2)}{p(z_1)p(z_2)} dx dz_1 dz_2\\
    &= \int p(x,z_1,z_2) \log \frac{p(x,z_1,z_2)}{p(x)}\cdot \frac{p(x)}{p(x,z_1)}\cdot \frac{p(x)}{p(x,z_2)} dz_1dz_2 dx  \\
    &~~~+ \int p(x,z_1) \log \frac{p(x,z_1)}{p(x)p(z_1)} dx dz_1 \\
    &~~~+ \int p(x,z_2) \log \frac{p(x,z_2)}{p(x)p(z_2)} dx dz_2\\
    &~~~- \int p(z_1,z_2) \log \frac{p(z_1,z_2)}{p(z_1)p(z_2)} dz_1 dz_2\\
    &= \int p(x)p(z_1,z_2|x) \log \frac{p(z_1,z_2|x)}{p(z_1|x) p(z_2|x)} dz_1dz_2 dx\\
    &~~~+ I(x;z_1)+I(x;z_2)-I(z_1;z_2)\\
    &= \mathbb{E}_{x\sim p(x)}[\int p(z_1,z_2|x) \log \frac{p(z_1,z_2|x)}{p(z_1|x) p(z_2|x)} dz_1dz_2]\\
    &~~~+ I(x;z_1)+I(x;z_2)-I(z_1;z_2)\\
    &= I(x;z_1)+I(x;z_2)-I(z_1;z_2)
\end{align*}
\end{proof}

\section*{B. Implementation details}

We follow the Network architecture in \cite{jointvae}. We use $[0,1]$ normalized image data.
\Cref{tab:arch_64} is the model architecture for $64\times 64$ images (Chairs and dSprites).
MNIST and FashionMNIST (which is $28\times 28$) is resized to $32\times 32$ and architecture in \Cref{tab:arch_32} was used.
Batch size for training is fixed with 64. $\beta_h$ is fixed with $10.0$ for our experiments.

\begin{table}[h]
\centering
\footnotesize
\begin{adjustbox}{max width=\columnwidth}
\begin{tabular}{l l}
\toprule
\textbf{Encoder} & \textbf{Decoder}\\
\midrule
$4 \times 4$ conv 32,ReLU, stride 2        & $\text{input dim} \times 256$ fully connected, ReLU\\ 
$4 \times 4$ conv 32,ReLU, stride 2        & $256 \times 64*4*4$ fully conncted, ReLU\\
$4 \times 4$ conv 64,ReLU, stride 2        & $4\times 4$ conv transpose 64, ReLU, stride 2\\ 
$4 \times 4$ conv 64,ReLU, stride 2        & $4\times 4$ conv transpose 32, ReLU, stride 2\\
$64*4*4\times256$ fully connected, ReLU    & $4\times 4$ conv transpose 32, ReLU, stride 2\\
$256\times \text{output dim}$ fully connected& $4\times 4$ conv transpose 1, Sigmoid, stride 2\\
\bottomrule
\end{tabular}
\end{adjustbox}
\label{tab:arch_64}
\caption{Encoder and decoder architecture for Dsprites and Chairs data}
\end{table}

\begin{table}[h]
\centering
\footnotesize
\begin{adjustbox}{max width=\columnwidth}
\begin{tabular}{l l}
\toprule
\textbf{Encoder} & \textbf{Decoder}\\
\midrule
$4 \times 4$ conv 32,ReLU, stride 2        & $\text{input dim} \times 256$ fully connected, ReLU\\ 
$4 \times 4$ conv 32,ReLU, stride 2        & $256 \times 64*4*4$ fully conncted, ReLU\\
$4 \times 4$ conv 64,ReLU, stride 2        & $4\times 4$ conv transpose 32, ReLU, stride 2\\ 
$64*4*4\times256$ fully connected, ReLU     & $4\times 4$ conv transpose 32, ReLU, stride 2\\
$256\times \text{output dim}$ fully connected& $4\times 4$ conv transpose 1, Sigmoid, stride 2\\
\bottomrule
\end{tabular}
\end{adjustbox}
\label{tab:arch_32}
\caption{Encoder and decoder architecture for MNIST and FashionMNIST}
\end{table}

\subsection{dSprites}
\begin{itemize}
\item Dimension of discrete : 3
\item Optimizer: Adam with learning rate 3e-4
\item $\lambda'$ : 0.001
\item $r$ : 2e4
\item $t_d$ : 1e5
\item Iterations : 3e5
\end{itemize}

\subsection{MNIST}
\begin{itemize}
\item Dimension of discrete : 10
\item Optimizer : Adam with learning rate 3e-4
\item $\lambda'$ : 0.1
\item $r$ : 1e4
\item $t_d$ : 0
\item Iterations : 1.2e5
\end{itemize}

\subsection{FashionMNIST}
\begin{itemize}
\item Dimension of discrete : 10
\item Optimizer : Adam with learning rate 1e-4
\item $\lambda'$ : 0.1
\item $r$ : 1e4
\item $t_d$ : 0
\item Iterations : 1.2e5
\end{itemize}

\subsection{Chairs}
\begin{itemize}
\item Dimension of discrete : 3
\item Optimizer: Adam with learning rate 1e-4
\item $\lambda'$ : 0.01
\item $r$ : 2e4
\item $t_d$ : 6e4
\item Iterations : 1.5e5
\end{itemize}

\section*{C. Disentanglement score}
We follow the disentanglement score details from \cite{factorvae} and \cite{jointvae}. At first, we prune out all latent dimensions where variational posterior collapses to the priror. Concretely, we prune the continous latent dimension $z_j$ where 
\begin{align*}
\mathbb{E}_{x \sim p(x)}D_\text{KL}(q_\phi(z_j\mid x)\parallel p(z_j))<0.1~~.
\end{align*}
We evaluate disentanglement score with the surviving dimensions. We choose a factor $k$ from $K$ factors (\ie~ postion x, position y, rotation, scale, shape). Then, we obtain the representations from $L$ ($=100$) data of which factor $k$ is fixed and the other factors are randomly chosen. We take the empirical variance of each latent dimensions and normalize with each empirical variance over the full data\footnote{\text{We denote the empirical variance of latent dimension $j$ on full data, $v_j$.}}.
Concretely, the empirical variance on $j$ latent dimension \footnote{For convenience, $z_{m+1}=d$.}, is defined as 
\begin{align*}
    \widehat{\text{Var}}_j = \frac{1}{2N(N-1)} \sum_{p,q=1}^N d(x_p, x_q),
\end{align*} 
where $d(x_p, x_q) = \begin{cases} \mathbb{I}(x_p\neq x_q)& \text{if $j\!=\! m+1$} \\ (x_p-x_q)^2 & \text{otherwise} \end{cases}$. 
This procedure generates a vote $(j,k)$ where 
\begin{align*}
    j=\argmin_{j^*} \frac{1} {v_{j^*}}\widehat{\text{Var}}_{j^*}.
\end{align*}
We generate $M$ ($=800$) votes $(a_i,b_i)_{i=1}^M$. Let $V_{jk}=\sum_{i=1}^M \mathbb{I}(a_i=j,b_i=k)$. Concretely, the disentanglement score is 
\begin{align*}
    \frac{1}{M} \sum_{j} \max_k V_{jk}.
\end{align*}
Random chance algoirhtm takes $\frac{1}{K}$ as a accuracy.

\end{document}